\documentclass[twoside,11pt]{article} 
\usepackage{jmlr2e}

\jmlrheading{}{2022}{1-55}{3/22}{}{rsonthal}{Anna C.~Gilbert and Rishi Sonthalia}

\ShortHeadings{Project and Forget}{R.~Sonthalia and A.~C.~Gilbert}
\firstpageno{1}

\usepackage{amsmath,amsfonts,bm}

\def\eqref#1{equation~\ref{#1}}

\def\1{\bm{1}}

\def\vone{{\bm{1}}}

\def\va{{\bm{a}}}
\def\vb{{\bm{b}}}

\def\vf{{\bm{f}}}
\def\vg{{\bm{g}}}

\def\mC{{\bm{C}}}

\def\mP{{\bm{P}}}

\DeclareMathAlphabet{\mathsfit}{\encodingdefault}{\sfdefault}{m}{sl}
\SetMathAlphabet{\mathsfit}{bold}{\encodingdefault}{\sfdefault}{bx}{n}

\def\gX{{\mathcal{X}}}
\def\gY{{\mathcal{Y}}}

\newcommand{\R}{\mathbb{R}}

\DeclareMathOperator*{\argmax}{arg\,max}
\DeclareMathOperator*{\argmin}{arg\,min}

\usepackage{hyperref}

\usepackage[T1]{fontenc}    
\usepackage{hyperref}       
\usepackage{booktabs}       
\usepackage{amsfonts}       
\usepackage{nicefrac}       
\usepackage{microtype}      
\usepackage{enumerate}
\usepackage{fancyhdr}
\usepackage{mathrsfs}
\usepackage{graphicx}
\usepackage{multirow}
\usepackage{enumitem}
\usepackage{chngcntr}
\usepackage{bbm}
\usepackage{subfigure}
\usepackage[noend]{algpseudocode}
\usepackage{enumitem}
\usepackage{algorithm,algorithmicx}
\usepackage{amsmath,xparse}
\usepackage{amssymb}
\usepackage[table,svgnames]{xcolor}
\usepackage{soul}

\counterwithin*{equation}{section}

\newcommand{\cellhi}{\cellcolor{RoyalBlue!40}}

\newcommand{\met}{\text{MET}}
\newtheorem{thm}{Theorem}
\newtheorem{lem}{Lemma}
\newtheorem{cor}{Corollary}
\newtheorem{prob}{Problem}
\newtheorem{defn}{Definition}
\newtheorem{prop}{Proposition}
\newtheorem{rem}{Remark}

\newtheorem{property}{Property}

\DeclareMathOperator*{\dist}{dist}

\newenvironment{reftheorem}[1]{\medskip\parindent 0pt{\bf Theorem \ref{#1}.}\em }{\vspace{1em}}

\newenvironment{refprop}[1]{\medskip\parindent 0pt{\bf Proposition \ref{#1}.}\em }{\vspace{1em}}

\numberwithin{equation}{section}

\begin{document}

\title{Project and Forget: Solving Large-Scale Metric Constrained Problems}

\author{%
  \name Rishi Sonthalia \\
  \addr Department of Mathematics\\
  University of Michigan\\
 Ann Arbor, Michigan, 48104\\
  \email rsonthal@umich.edu \\
  \AND
 \name Anna C.~Gilbert \\
  \addr Departments of Mathematics and Statistics \& Data Science\\ Yale University \\
 New Haven, Connecticut, 06510\\
  \email anna.gilbert@yale.edu }

\editor{}

\maketitle

\begin{abstract}

Many important machine learning problems can be formulated as highly constrained convex optimization problems. One important example is metric constrained problems. In this paper, we show that standard optimization techniques can not be used to solve metric constrained problem. 

To solve such problems, we provide a general active set framework, called \textsc{Project and Forget}, and several variants thereof that use Bregman projections. \textsc{Project and Forget} is a general purpose method that can be used to solve highly constrained convex problems with many (possibly exponentially) constraints. We provide a theoretical analysis of \textsc{Project and Forget} and prove that our algorithms converge to the global optimal solution and have a linear rate of convergence. We demonstrate that using our method, we can solve large problem instances of general weighted correlation clustering, metric nearness, information theoretic metric learning and quadratically regularized optimal transport; in each case, out-performing the state of the art methods with respect to CPU times and problem sizes. 


\end{abstract}




\section{Introduction}


Given a set of dissimilarity measures amongst data points, many machine learning problems are considerably ``easier'' if these dissimilarity measures adhere to a metric. Furthermore, learning the metric that is most ``consistent'' with the input dissimilarities or the metric that best captures the relevant geometric features of the data (e.g., the correlation structure in the data) is a key step in efficient, approximation algorithms for classification, clustering, regression, and feature selection. In practice, these metric learning problems are formulated as convex optimization problems subject to metric constraints, such as the triangle inequality, on all the output variables. Because of the large number of constraints, researchers have been forced to restrict either the kinds of metrics learned or the size of the problem that can be solved. In many cases, researchers have restricted themselves to learning (weighted) Euclidean or Mahalanobis metrics. This approach is, however, far from ideal as the inherent geometry of many data sets necessitates different types of metrics. Therefore, we need to develop optimization techniques that can optimize over the space of all metrics on a data set. 

Many of the existing standard optimization methods suffer from significant drawbacks when solving metric constrained problems that hamper performance and restrict the instance size.  Gradient based algorithms such as projected gradient descent (e.g., \citet{Beck2009AFI,Nesterov1983AMF}) or Riemannian gradient descent require a projection onto the space of all metrics which, in general, is an intractable problem. One modification of this approach is to sub-sample the constraints and then project onto the sampled set (see \citet{Nedic2011RandomAF, Polyak2001RandomAF, Wang2013IncrementalCP, Wang2015RandomMP}). For metric constrained problems, however, there are many more constraints than data points, so the condition numbers of the problems are quite high and, as a result, these algorithms tend to require a large number of iterations. 

Another standard general optimization technique is based on the Lagrangian method. These algorithms augment the objective (or introduce a barrier function) by adding a term for each constraint.  Examples of such methods include the interior point method, the barrier method, and the Alternating Direction Method of Multipliers. These methods run into two different kinds of problems. First, computing the gradient becomes an intractable problem because of the large number of constraints. One fix could be to sub-sample the constraints and compute only those gradients, but this approach runs into the same drawbacks as we discussed before. The other option is to incrementally update the Lagrangian, looking at one constraint at a time. Traditionally, these methods require us to cycle through all the constraints. One such method is Bregman cyclic method and we note that the requirement to examine the constraints cyclically is an aspect that is highlighted in various previous works \cite{Censor98thedykstra, dykstra, book}.  This is time consuming with metric constraints as we have cycle through at least $O(n^3)$ many constraints. Many applications that use this method sidestep the issue either by restricting the number of constraints and solving a heuristic problem \citep{itml}, by solving smaller sized problems \citep{Dhillon2007MatrixNP}, or by trying to parallelize the projections \citep{veldt2}. 

A third approach is to use traditional active set methods such as Sequential Linear (Quadratic) Programming \citep{slp,sqp} or the algorithm from \cite{SCHITTKOWSKI20092999}. In general, these methods maintain a set of constraints $C$ that is assumed to be the true set of active constraints. During each iteration, they \emph{completely} solve the sub-problem defined by the constraints in $C$. They then check which of the constraints in $C$ are inactive and remove those constraints. They also search for and add new violated constraints. These methods run into the problem that each iteration is computationally expensive and, in some cases, may not be tractable because of the large number violated constraints. If the size of $C$ is reduced, then the number of iterations becomes too large, again making the problem intractable. 

One final approach is to use cutting planes. The performance of this method is heavily dependent on the cut selection process (see \citet{Dey2018TheoreticalCT, Poirrier2019OnTD} for deep discussions). The discovery of Gomory cuts \citep{gomory1960algorithm} and other subsequent methods such as branch and bound, has led to the viability of the cutting plane method for solving mixed integer linear programs. This success has, however, not transferred to other problems. In general, if the cuts are not selected appropriately, the algorithm could take an exponential number of iterations; i.e., it might add an exponential number of constraints. To use this method, we must show for each problem that the specific cutting plane selection method results in a feasible algorithm (see \citet{Chandrasekaran2012TheCP} for an example). 

We note that the only method that has found any success for metric constrained problems is the Lagrangian based method known as the cyclic Bregman method \citep{BREGMAN1967200}. This was first used by \cite{Brickell2008TheMN} to solve metric constrained problems. Follow up work such as \cite{itml, veldt, veldt2} also used this methods or variants of this methods to solve metric constrained problems. Although, we note the significant drawbacks to this approach above.

In this paper, we provide an active set algorithm, \textsc{Project and Forget}, that uses Bregman projections to solve \textbf{convex optimization problems with a large number of (possibly exponentially linear inequality) constraints}. Since our algorithm is based on the cyclic Bregman method, it has the rapid rate of convergence of the Bregman method, along with all the benefits of being an active set method. This method overcomes the weaknesses of both the traditional active set methods and Bregman cyclic method. First, we overcome the drawbacks of the active set methods by allowing the introduction of new and the removal of old constraints \emph{without having to completely solve convex programs as intermediate steps}. In particular, our algorithm examines each constraint once, before it introduces new constraints and forgets old constraints, thus allowing us to converge rapidly to the true active constraint set. We overcome the drawback of the traditional Bregman method by cycling through the current active set only. Thus, making each iteration much faster. Our new algorithm \textsc{Project and Forget}, is the first iterative Bregman projection based algorithm for convex programs that does not require us to cyclically examine all the constraints. 


The major contributions of our paper are as follows:

\begin{enumerate}[nosep]
\item For the case when we have linear inequality constraints, we provide a Bregman projection based algorithm that does not need to look at the constraints cyclically. We prove that our algorithm converges to the global optimal solution and that the optimality error ($L_2$ distance of the current iterate to the optimal) asymptotically decays at an exponential rate. We also show that because of the \textsc{Forget} step, when the algorithm terminates, the set of constraints remembered are exactly the active constraints.

\item For the case when we have general convex constraints, we provide a Bregman projection based algorithm that does not need to look at the constraints cyclically. We prove that our algorithm converges to the global optimal solution and that when we have quadratic objective function, the optimality error ($L_2$ distance of the current iterate to the optimal) asymptotically decays at an exponential rate.  We also show that because of the \textsc{Forget} step, when the algorithm terminates, the set of constraints remembered are exactly the active constraints.

\item We solve the weighted correlation clustering problem \cite{Bansal2004} on a graph with over $130,000$ nodes. To solve this problem with previous methods, we would need to solve a linear program with over $10^{15}$ constraints. Furthermore, we demonstrate our algorithms superiority by outperforming the current state of the art in terms of CPU times.

\item We use our algorithm to develop a new one that solves the metric nearness problem \cite{Brickell2008TheMN}. We show that our algorithm outperforms the current state of the art with respect to CPU time and can be used to solve the problem for non-complete graphs. 

\item We also show the generality of our algorithm by using it to solve the quadratically regularized optimal transport problem. We show that, using our algorithm, we can solve this problem faster and using less memory than many standard solvers.


\end{enumerate}

\section{Preliminaries} \label{sec:prelims}

We start by presenting the general version of the problem before focusing on metric constrained problems in later sections.

\subsection{Convex Programming}  \label{sec:generalProblem}

Given a strictly convex function $f: \mathbb{R}^d \to \mathbb{R}$,  and a finite family of convex sets $\mathcal{F} = \{C_i\}$ we want to find the unique point $x^* \in \bigcap_i C_i =: C$ that solves the following problem. 
\begin{equation}
\label{problem:general}
\begin{array}{ll@{}ll}
\text{minimize}  &f(x)  \\
\text{subject to}& \forall i, \, x \in C_i.\\
\end{array}
\end{equation}

We refer to each $C_i$ as a \emph{constraint set} and $C := \bigcap_i C_i$ as the \emph{feasible region}. We shall assume that $C$ is not empty; i.e., there is at least one feasible point. Since we have a large number of constraint sets, we access the constraint sets only through an oracle that has one of the two following separation properties. 

\begin{property} \label{prop:sep} $\mathcal{Q}$ is a deterministic separation oracle for a family of convex sets $\mathcal{F} = \{ C_i\}$, if there exists a non-decreasing, continuous function $\phi$, with $\phi(y) = 0 \iff y = 0$, such that on input $x \in \mathbb{R}^d$, $\mathcal{Q}$ either certifies $x \in C$ or returns a list $\mathcal{L} \subset \mathcal{F} $ such that \[ \max_{\tilde{C} \in \mathcal{L}} \dist(x,\tilde{C}) \ge \phi(\dist(x,C)), \] where for a point $x$ and set $B$, $\dist(x,B) = \inf_{w \in B} \|w-x\|$. \end{property}

There a few things that we would like to highlight about this definition. First, the list $\mathcal{L}$ \emph{need not contain all violated constraints.} That is, given $x$, there can be some $C_i \in \mathcal{F}$ such that $x \notin C_i$ \emph{and} $C_i \not\in \mathcal{L}$. In fact, there could be many such $C_i$. However, what we do require is that the maximum distance from $x$ to the constraints in $\mathcal{L}$ is at least some non-decreasing function $\phi$ of the distance from $x$ to $C$.

\begin{property} \label{prop:rand} $\mathcal{Q}$  is a random separation oracle for a family of convex sets $\mathcal{F}$, if there exists a lower bound $\tau > 0$, such that on input $x \in \mathbb{R}^d$, $\mathcal{Q}$ returns a list $\mathcal{L} \subset \mathcal{F} $ such that 
\[ 
	\forall \tilde{C} \in \mathcal{F},\, Pr [\tilde{C} \in \mathcal{L}] \ge \tau. 
\] 
\end{property}

\begin{rem} The random separation oracle need not decide whether $x \in C$. \end{rem}

\subsubsection{Linear Inequality Constraints}

In practice, most problems do not have the most general of convex constraints. Indeed, linear inequality constraints are common. In such a case, our constraints sets $C_i$ are half spaces. In particular, we denote each half space by $H_i$ instead of $C_i$. Additionally, for each $H_i$, we know that there exists $a_i \in \mathbb{R}^d$ and $b_i \in \R$ such that 

\[
    H_i = \{ x \in \R^d : \langle a_i, x \rangle \le b_i \}.
\]
\noindent Thus, if $A$ is the matrix whose rows are given by $a_i$ and $b$ is the vector whose coordinates are $b_i$, then our feasible region $C$ can be represented as follows:
\[
    C = \{ x \in \R^d : Ax \le b \}.
\]

\noindent In this case, we can reformulate Problem \ref{problem:general} as follows. 
\begin{equation}
\label{problem:linear}
\begin{array}{ll@{}ll}
\text{minimize}  &f(x)  \\
\text{subject to}& Ax \le b\\
\end{array}
\end{equation}

Solving this problem for general convex functions $f$ is too monumental a task. We restrict ourselves to a rich class of functions known as Bregman functions. First, we define the generalized Bregman distance.

\begin{defn} \label{defn:bdist} Given a convex function $f(x): S \to \R$ whose gradient is defined on $S$, we define its \emph{generalized Bregman distance} $D_f : S \times S \to \R$ as $D_f(x,y) = f(x) - f(y) - \langle\nabla f(y), x-y \rangle. $
\end{defn}

\begin{defn} \label{defn:bfunction} 
A function $f: \Lambda \to \mathbb{R}$ is called a Bregman function if there exists a non-empty convex set $S$ such that $\overline{S} \subset \Lambda$ and the following hold: 
\begin{enumerate}[label = (\roman*), topsep=0pt,itemsep=-1ex,partopsep=1ex,parsep=1ex]
\item $f(x)$ is continuous, strictly convex on $\overline{S}$, and has continuous partial derivatives in $S$.
\item For every $\alpha \in \mathbb{R}$, the partial level sets $\displaystyle L_1^f(y, \alpha) := \{ x \in \overline{S} : D_f(x,y) \le \alpha \} $ and $\displaystyle L_2^f(x, \alpha) := \{ y \in S : D_f(x,y) \le \alpha \} $ are bounded for all $x \in \overline{S}, y \in S$. 
\item If $y_n \in S$ and $\displaystyle \lim_{n \to \infty} y_n = y^*$, then $\displaystyle \lim_{n \to \infty} D_f(y^*,y_n) = 0$.
\item If $y_n \in S$, $x_n \in \overline{S}$, $\displaystyle \lim_{n \to \infty} D_f(x_n,y_n) = 0$,  $y_n \to y^*$, and $x_n$ is bounded, then $x_n \to y^*$.
\end{enumerate} 
We denote the family of Bregman functions by $\mathcal{B}(S)$. We refer to $S$ as the zone of the function and we take the closure of the $S$ to be the domain of $f$. Here $\overline{S}$ is the closure of $S$. 
\end{defn}

This class of function includes many natural objective functions, including entropy  $f(x) = -\sum_{i=1}^n x_i \log(x_i)$ with zone $S = \mathbb{R}^n_{+}$ (here $f$ is defined on the boundary of $S$ by taking the limit) and $f(x) = \frac{1}{p} \|x\|_p^p$ for $p \in (1, \infty)$. The $\ell_p$ norms for $p=1,\infty$ are not Bregman functions but can be made Bregman functions by adding a quadratic term. That is, $f(x) = c^Tx$ is a not Bregman function, but $c^Tx + x^TQx$ for any positive definite $Q$ is a Bregman function. 

\begin{defn} We say that a hyper-plane $H_i$ is \emph{strongly zone consistent} with respect to a Bregman function $f$ and its zone $S$, if for all $y\in S$ and for all hyper-planes $H$, parallel to $H_i$ that lie in between $y$ and $H_i$, the Bregman projection of $y$ onto $H$ lies in $S$ instead of in $\overline{S}$. 
\end{defn} 

In addition to the restrictions on the functions for which we can solve Problem \ref{problem:linear}, we will need the assumption that all hyper-planes in $\mathcal{H}$ (our family of half spaces) are strongly zone consistent with respect to $f(x)$. This assumption is used to guarantee that when we do a projection the point we project onto lies within our domain.  This is also not too restrictive. For example, all hyper-planes are strongly zone consistent with respect to the objective functions $f(x) = 0.5\|x\|^2$ and $f(x) = -\sum_i x_i\log(x_i)$. The final assumption, that we mentioned earlier, is that $C$ is non-empty. This is needed to ensure the algorithm converges. 

\subsubsection{General Convex Constraints}
\label{sec:general-convex-constraints}
While general convex constraints do not often appear in practical problems, such a formulation is of theoretical interest. When our constraints are simply convex sets rather than linear ones, we must adjust our algorithmic approach and, as a result, our assumptions about the objective function $f(x)$ and the constraints will be slightly different as compared to the linear case. The problem that we are interested in is
\begin{equation}
\label{problem:convex}
\begin{array}{ll@{}ll}
\text{minimize}  &f(x)  \\
\text{subject to}& x \in C_i &\ \ \  \forall \, C_i.\\
\end{array}
\end{equation}

Before we state our assumptions, we need the following definitions. 

\begin{defn} A function $f$ is Legendre, if $f$ is a closed proper map, $int(dom f) \neq \emptyset$, and $ \lim_{t \downarrow 0} \langle \nabla f(x+t(y-x)), y-x\rangle = -\infty$ for all $x \in \partial(dom f)$ and for all $y \in int(dom f)$. \end{defn}

\begin{defn} A function $f$ is strictly convex, if $f$ is twice differentiable everywhere and its Hessian is positive definite everywhere. \end{defn}
\begin{defn} A function $f$ is co-finite if $\lim_{r \to \infty} f(rx)/r = \infty$ for all $x \in dom f$. \end{defn}

For this version of the problem we will assume that $f$ is a strictly convex, co-finite, Legendre function. Note that these are strong conditions and rule out some objective functions such as Burg's entropy $f(x) = \sum_{i} \log(x_i)$. They do still, however, allow a rich class of functions. Some examples can be seen below:
\begin{enumerate}
\item $f(x) = 0.5 \|x\|^2$ on $\mathbb{R}^n$.
\item $f(x) = \sum_i x_i\log(x_i) - x_i$ (Boltzman/Shannon entropy) on $\mathbb{R}^n_+$.
\item $f(x) = - \sum_i\sqrt{1-x_i^2}$ (Hellinger distance) in $[-1,1]^n$.
\item $f(x) = \sum_i x_i \log(x_i) + (1-x_i)\log(1-x_i)$ (Fermi/Dirac entropy) on $[0,1]^n$.
\item $f(x) = \begin{cases} \frac{1}{2}x^2 +2x + \frac{1}{2} & x \le -1 \\ -1 - \log(-x) & -1 \le x < 0 \end{cases}$.
\end{enumerate}

Note that the last example Legendre function is not a Bregman function. In addition to the Legendre function assumptions, we also assume that $int(dom f) \cap C \neq \emptyset$. See \cite{dykstra} for a discussion comparing the two different classes functions presented. 

\subsection{Metric Constrained Problems}

The primary motivation of this paper is to solve metric constrained problems and in this section we set up such problems.
Prior work such as \cite{veldt, Brickell2008TheMN} define a metric constrained problem as follows. Let $x \in \mathbb{R}^{n \choose 2}$ so that the entries are indexed as $x_{ij}$ for $i < j$. Let $f: \mathbb{R}^{n \choose 2} \to \mathbb{R}$ be the function to optimize. Then a metric constrained problem, as defined in prior work, is the following optimization problem. 
\begin{equation}
\label{problem:metric}
\begin{array}{ll@{}ll}
\text{minimize}  &f(x)  \\
\text{subject to}& \forall\ i<j<k, \ x_{ij} \le x_{ik} + x_{jk} 
\end{array}
\end{equation}
In this paper, however, we will generalize this problem. 

To define general metric constrained problems, we first define the metric polytope. 

\begin{defn} Let $\text{MET}_n \subset \R^{\binom{n}{2}}$ be the space of all pseudo-metrics on $n$ points. Given a graph $G$ the metric polytope $\text{MET}_n(G)$ is the projection of $\text{MET}_n$ onto the coordinates given by the edges of $G$ (i.e., we consider distances only between pairs of points that are adjacent in $G$). \end{defn}

It can be easily seen that for any $x \in \mathbb{R}^{\binom{n}{2}}$, $x \in \text{MET}_n(G)$ if and only if $\forall\, e\in G,\, x(e) \ge 0$ and for every cycle $\mathcal{C}$ in $G$ and  $\forall\, e \in \mathcal{C}$, we have that  \[ x(e) \le \sum_{\tilde{e} \in \mathcal{C}, \tilde{e} \neq e} x(\tilde{e}).\] Therefore, $\text{MET}_n(G)$ can be described as the intersection of exponentially many half-spaces. 

\begin{rem} It is important to note that $\text{MET}_n$ is the space of \emph{all} metrics on $n$ points. Hence, when we optimize over $\text{MET}_n$ (or over $\text{MET}_n(G)$), we are optimizing over a much larger and more complex space than the space of Euclidean metrics or the space of all Mahalanobis metrics. 
\end{rem}

Now that we have the set over which we want to optimize, we give a general formulation for metric constrained optimization problems.
\begin{defn}
Given a strictly convex function $f$, a graph $G$, and a finite family of half-spaces $\mathcal{H} = \{H_i\}$ such that $H_i = \{ x : \langle a_i , x\rangle \le b_i \}$, we seek the unique point $x^* \in \bigcap_i H_i \cap \text{MET}(G) =: C$ that minimizes $f$. That is, if we set $A$ to be the matrix whose rows are $a_i$ and $b$ be the vector whose coordinates are $b_i$ we seek
\begin{equation}
\label{problem:metric}
\begin{array}{ll@{}ll}
\text{minimize}  &f(x)  \\
\text{subject to}& Ax \le b\\
& x \in MET(G).
\end{array}
\end{equation}
\end{defn}
The constraints encoded in the matrix $A$ let us impose additional constraints, beyond the metric constraints. For example, in correlation clustering, the matrix $A$ encodes $x_{ij} \in [0,1]$. In general, we will assume that the number of additional constraints encoded in $A$ (beyond the metric constraints) is relatively small so that the predominant difficulty in solving these optimization problems comes from the metric constraints. 

It is clear if $G = K_n$, then $\met_n = \met(K_n)$. Therefore, the metric constrained problems from \cite{veldt, Brickell2008TheMN} are simply special cases of our problem. There is, however, an equivalence beyond this as well when $f$ does not depend on some of the coordinates of $x$. For example, suppose $n = 3$, $x = (x_{12}, x_{13}, x_{23})$, and $f(x) = x_{12} + x_{23}$. Then, if we optimize $f$ over $\met_3$ or $\met(G)$ where $G$ is the path $1-2-3$, then the solutions are the same. This is formalized by the following proposition. 

\begin{prop} \label{prop:fiber} Let $f(x)$ be a function whose values only depends on the values $x_{ij}$ for $e= (i,j) \in G$ and consider the following constrained optimization problem. 
\begin{equation}
\begin{array}{ll@{}ll}
\text{minimize}  & f(x) & \\
\text{subject to}& x \in \text{MET}(K_n). \\
\end{array}
\label{problem:orig}
\end{equation}
Let $\pi$ be the projection from MET$(K_n)$ to MET$(G)$ and let $\tilde{f}(x) := f(\pi^{-1}(x))$. Then for any optimal solution $x^*$ to the following problem \begin{equation}
\begin{array}{ll@{}ll}
\text{minimize}  & \tilde{f}(x) & \\
\text{subject to}& x \in \text{MET}(G). \\
\end{array}
\label{problem:smaller}
\end{equation}
we have that for all $\hat{x} \in \pi^{-1}(x^*)$, $\hat{x}$ is an optimal solution to \ref{problem:orig}.
\end{prop}

\begin{proof} Here, we see that if $\tilde{x}$ is a minimizer of Problem \ref{problem:orig} and $x^*$ is the minimizer of Problem \ref{problem:smaller} then \[ f(\tilde{x}) = \tilde{f}(\pi(\tilde{x})) \ge \tilde{f}(x^*) = f(\pi^{-1}(x^*)), \] where the middle inequality is true, since $\pi(\tilde{x})$ need not be an optimal solution to Problem \ref{problem:smaller}. Thus, any element in $\pi^{-1}(x^*)$ is an optimal solution to Problem \ref{problem:orig}. \end{proof}

In these cases, we see that while we are solving the same problem, there are some practical differences. When $\met_n = \met(K_n)$, we are looking at two different representations of the same polytope; one with $O(n^3)$ constraints and the other with significantly more. In the case when $f$ is constant for some coordinates of $x$, we have another practical difference; that is, we compute $x^*$ the optimal point in $\met(G)$. It is, however, unclear how to pick the point $\hat{x} \in \pi^{-1}(x^*)$.

\subsection{Projections}

All of our algorithms will be based on iteratively computing Bregman projections.
\begin{defn} \label{defn:proj} Given a strictly convex function $f$, a closed convex set $Y$, and a point $y$, the projection of $y$ onto $Y$ with respect to $D_f$ is a point $x^* \in {\rm dom}(f)$ such that 
\[ x^* = \argmin_{x \in Y \cap {\rm dom}(f)} D_f(x,y). \] 
\end{defn}

In the case we have linear inequality constraints, we project onto the boundary of the half space $\partial H$. In this case, the Bregman projection has some additional special properties. 

\begin{lemma}\label{lem:bregman-facts} Let $x$ be the point that we project onto $\partial H_i = \{ y \in \R^d : \langle y, a_i \rangle = b_i\}$, then there exists a unique $x^*, \theta$ such that $\nabla f(x^*)= \nabla f(x) + \theta a_i$ and $\langle x^*, a_i \rangle = b_i$. This unique $x^*$ is also the Bregman projection of $x$ on $\partial H_i$. Furthermore,
\begin{enumerate}
\item $\theta > 0$ if and only if $\langle x, a_i \rangle > b_i$;
\item $\theta < 0$ if and only if $\langle x, a_i \rangle < b_i$;
\item $\theta = 0$ if and only if $\langle x, a_i \rangle = b_i$.
\end{enumerate} 
\end{lemma}

\section{Project and Forget: Linear Inequalities}
\label{sec:Algorithm-linear}

To set the stage for subsequent discussions, we present the general structure of our algorithm first and then detail adjustments we make for the different kinds of constraints (linear inequalities versus the more general convex constraints). It is iterative and, in general, will be run until some convergence criterion has been met. The convergence criterion depends largely on the specific application for which the algorithm is tailored. For this reason, we postpone the discussion of the convergence criterion until the applications section.

The \textsc{Project and Forget} algorithm keeps track of three quantities; $x^{(\nu)}$, the vector of variables over which we optimize, $L^{(\nu)}$ a list of the constraints that the algorithm deems active, and $z^{(\nu)}$ a vector of dual variables. Each iteration of the \textsc{Project and Forget} algorithm consists of three phases. In the first phase, we query our oracle $\mathcal{Q}$ to obtain a list of constraints $L$. In the second phase, we merge $L^{(\nu)}$ with $L$ to form $\tilde{L}^{(\nu+1)}$ and project onto each of the constraints in $\tilde{L}^{(\nu+1)}$ one at a time. When we do these projections, we update $x^{(\nu)}$ and $z^{(\nu)}$. Finally, in the third phase, we forget some constraints from $\tilde{L}^{(\nu+1)}$ to yield $L^{(\nu+1)}$.

\begin{algorithm}[!ht]
\caption{General Algorithm.}
\label{alg:bregman}
\begin{algorithmic}[1]
\Function{Project and Forget}{$f$ convex function}
    \State $L^{(0)} = \emptyset$, $z^{(0)} = 0$. Initialize $x^{(0)}$ so that $\nabla f(x^{(0)}) = 0$.
    \While{Not Converged}
        \State $L = \mathcal{Q}(x^{\nu})$
        \State $\tilde{L}^{{(\nu+1)}} = L^{(\nu)} \cup L$
        \State $x^{(\nu+1)}, z^{(n+1)} = $ Project($x^{(\nu)}, z^{(\nu)}, \tilde{L}^{(\nu+1)}$) 
        \State $L^{(\nu+1)} = $ Forget($z^{(\nu+1)}, \tilde{L}^{(\nu+1)}$) 
    \EndWhile
    \Return $x$
\EndFunction
\end{algorithmic}
\end{algorithm}

\subsection{Finding Violated (Metric) Constraints}
\label{sec:findcosntraints}

The first step of the method is to find violated constraints and in this subsection we detail how to find violated metric constraints in particular (which are a special case of linear inequality constraints). In many applications, we could do this by searching through the list of constraints until we found a violated constraint. However, in our case, since $\text{MET}_n(G)$ has exponentially many faces, we cannot list all of them, so we seek an efficient separation oracle $\mathcal{Q}$. That is, given a point $x$, the oracle efficiently return a list $L$ of violated constraints, such that the constraints in $L$ satisfy some properties. We will assume that $\mathcal{Q}$ satisfies either the Property \ref{prop:sep} or Property \ref{prop:rand}. 

\begin{algorithm}[!ht]
\caption{Finding Metric Violations.}
\label{alg:apsp}
\begin{algorithmic}[1]
    \Function{\textsc{Metric Violations}}{$d$}
	\State $L = \emptyset$
	\State Let $d(i,j)$ be the weight of shortest path between nodes $i$ and $j$ or $\infty$ if none exists. 
	\For{Edge $e = (i,j) \in E$}
    	\If{$w(i,j) > d(i,j)$}
        		\State Let $P$ be the shortest path between $i$ and $j$
        		\State Add $C = P \cup \{(i,j)\}$ to $L$
   	\EndIf
	\EndFor
	\Return $L$
\EndFunction
\end{algorithmic}
\end{algorithm}

For metric constrained problems, Algorithm \ref{alg:apsp} finds violated constraints. If the metric constrained problem has additional constraints (i.e $Ax \le b$), then we augment our oracle accordingly. 

\begin{prop} \label{prop:metricoralce} \textsc{Metric Violation} runs $\Theta(n^2\log(n) + n|E|)$ time and satisfies Property \ref{prop:sep} with $\phi(y) = \frac{y}{n^{1.5}}$. \end{prop}

\begin{proof} The first step in \textsc{Metric Violation} is to calculate the shortest distance between all pairs of nodes. This can be done using Dijkstra's algorithm in $\Theta(n^2\log(n) + n|E|)$ time.  Then, if the shortest path between any adjacent pair of vertices is not the edge connecting them, then the algorithm has found a violated cycle inequality.  Note that if no such path exists, then all cycle inequalities have been satisfied and the input point $x$ (representing distances) is within the metric polytope. Thus, we have an oracle that separates the polytope.  

However, we want an oracle that also satisfies property \ref{prop:sep}. To that end, let us define the deficit of a constraint. Given a point $x$ and a hyper-plane $H_{\mathcal{C},e}$, defined by some cycle $\mathcal{C}$ and an edge $e$, the deficit of this constraint is given by
\[ 
    d(\mathcal{C},e) =  x(e) - \sum_{\tilde{e} \in \mathcal{C}, \tilde{e} \neq e} x(\tilde{e}). 
\] 
If this quantity is positive, then $x$ violates this constraint. In this case, the squared distance from $x$ to this constraint is $\frac{d(\mathcal{C},e)^2}{|\mathcal{C}|}$ (i.e., we add/subtract $\frac{d(\mathcal{C},e)}{|\mathcal{C}|}$ to each edge weight of $\mathcal{C}$).

Now let $x_{{\rm apsp}}$ be the all pair shortest path metric obtained from $x$ and $\mathcal{L}$ be the list returned by the oracle. Then
\[
	\|x-x_{{\rm apsp}}\|_2^2 = \sum_{\mathcal{C}, e \in \mathcal{L}} d(\mathcal{C},e)^2.
\]
Thus, if $H_{\tilde{\mathcal{C}},\tilde{e}}$ is the constraint that maximizes $d(\tilde{\mathcal{C}},\tilde{e})$, then we have that 
\[
	\dist(x,\tilde{C}_{\tilde{\mathcal{C}},\tilde{e}})^2 = \frac{d(\tilde{\mathcal{C}},\tilde{e})^2}{|\tilde{\mathcal{C}}|} \ge \frac{\|x-x_{{\rm apsp}}\|_2^2 }{|\tilde{\mathcal{C}}||\mathcal{L}|}.
\]
Since our oracle returns at most 1 constraint per edge, we have that $|\mathcal{L}| \le |E| \le n^2$. This along with the fact that $|\tilde{\mathcal{C}}| \le n$, gives us that 
\[
	\dist(x,H_{\tilde{\mathcal{C}},\tilde{e}})^2 \ge \frac{\|x-x_{{\rm apsp}}\|_2^2 }{n^3}.
\]
\noindent Finally, we know that $x_{{\rm apsp}} \in \met(G)$. Thus, we see that 
\[
	\|x_{{\rm apsp}} - x\|_2^2 \ge \dist(x,\met(G))^2 =  \dist(x,C)^2.
\]
\noindent Putting it all together, we have that 
\begin{align*}
	\max_{\hat{C} \in \mathcal{L}} \dist(x,\hat{C})^2 &\ge \dist(x,H_{\tilde{\mathcal{C}},\tilde{e}})^2 \\
	&\ge \frac{\|x-x_{{\rm apsp}}\|_2^2 }{n^3} \\
	&\ge \frac{\dist(x,C)^2}{n^3}.
\end{align*}
Taking the square root of both sides gives the needed result. 
\end{proof}

This oracle is the reason we can use the version of the metric polytope with exponentially many constraints rather than the one with cubically many constraints. This oracle runs in sub-cubic time in many cases; hence, we find violated constraints quickly. Note the oracle also returns at most $O(n^2)$ constraints. Thus, if we have sub-cubically many active constraints, this version does many fewer projections.  

\subsection{Project and Forget Steps} 
\label{sec:projectandforget}

The Project and Forget steps for the algorithm are presented in Algorithm \ref{alg:pf1}. Let us step through the code to obtain an intuitive understanding of its behavior. Let $H_i = \{ x : \langle a_i, x \rangle \le b_i\}$ be a constraint and $x$ the current iterate. The first step is to calculate $x^*$ and $\theta$. Here $x^*$ is the projection of $x$ onto the boundary of $H_i$ and $\theta$ is a ``measure'' of how far $x$ is from $x^*$. In general, $\theta$ can be any real number and so we examine two cases: $\theta$ positive or negative. 

It can be easily seen from Lemma \ref{lem:bregman-facts} that $\theta$ is negative if and only if the constraint is violated. In this case, we have $c = \theta$ because (as we will see in proof) the algorithm always maintains $z_i \ge 0$. Then on line 5, we compute the projection of $x$ onto $H_i$. Finally, since we corrected $x$ for this constraint, we add $|\theta|$ to $z_i$. Since each time we correct for $H_i$, we add to $z_i$, we see that $z_i$ stores the total corrections made for $H_i$. On the other hand, if $\theta$ is positive, this constraint is satisfied. In this case, if we also have that $z_i$ is positive; i.e., we have corrected for $H_i$ before and we have over compensated for this constraint. Thus, we must undo some of the corrections. If $c = z_i$, then we undo all of the corrections and $z_i$ is set to 0. Otherwise, if $c= \theta$ we only undo part of the correction.  

For the Forget step, given a constraint $H_i \in \tilde{L}^{(\nu+1)}$, we check if $z_i^{(\nu+1)} = 0$. If so, then we have not done any net corrections for this constraint and we can forget it; i.e., delete it from $\tilde{L}^{(\nu+1)}$. 

If we think of $L^{(\nu)}$ as matrix, with each constraint being a row, we see that at each iteration $L^{(\nu)}$ is a sketch of the matrix of active constraints. Hence, during each iteration we update this sketch by adding new constraints (rows). During the Forget step, we determine which parts of our sketch are superfluous and we erase (forget) these parts (rows) of the sketch.

\begin{algorithm}[!ht]
\caption{Project and Forget algorithms.}
\label{alg:pf1}
\begin{algorithmic}[1]
\Function{\textsc{Project}}{$x,z,L$}
    \For{ $H_i = \{ y: \langle a_i,y \rangle = b_i\} \in L$}
    	\State Find $x^*, \theta$ by solving $\nabla f(x^*) - \nabla f(x) = \theta a_i \text{ and } x^* \in H_i$
        \State  $c_i = \min\left(z_i,\theta \right)$
        \State $x \leftarrow x_{new}$, $x_{new}$ $\leftarrow$ such that $\nabla f(x_{new}) - \nabla f(x) = c_i a_i$
        \State $z_i \leftarrow z_i - c_i$
    \EndFor
    \Return $x$, $z$ 
\EndFunction
\Function{\textsc{Forget}}{$z,L$}
    \For{ $H_i = \{ x: \langle a_i,x \rangle = b_i\} \in L$}
        \If{$z_i == 0$}  Forget $H_i$ 
        \EndIf
    \EndFor
    \Return $L$
\EndFunction
\end{algorithmic}
\end{algorithm}

In general, calculating the Bregman projection (line 3) may not have a closed form formula. See \citet{Dhillon2007MatrixNP} for a general method to perform the calculation on line 3. For example, if  $f(x) = x^T Q x + r^T x + s$ where $Q$ positive definite, then for a given hyper-plane $\langle a, x \rangle = b$ and a point $x$ we have that 
\begin{equation} \label{eq:quadratic}
\theta = \frac{\langle a,x\rangle - b}{a^T Q^{-1} a}.
\end{equation}

\subsection{Truly Stochastic Variant} 
\label{sec:stochastic}

In some problems, we have constraints defined using only subsets of the data points and we may not have an oracle that satisfies Property \ref{prop:sep}. For such cases, we present a stochastic version of our algorithm. Instead of calling \textsc{Metric Violation} or an oracle with Property \ref{prop:sep}, we want an oracle with Property \ref{prop:rand}. This version of our algorithm is very similar to the algorithms presented in \cite{Nedic2011RandomAF, Wang2015RandomMP}. The major difference being that we do not need to perform a gradient descent step. Instead, we maintain the KKT conditions by keeping track of the dual variables and doing dual corrections. In practice, using \textsc{Project and Forget} with the random oracle tends to produce better results than~\cite{Nedic2011RandomAF,Wang2015RandomMP} because we remember the active constraints that we have seen, instead of hoping that we sample them.

In some cases, we may want a more stochastic variant. With the algorithm as specified, we have to keep track of the constraints that we have seen and carefully pick which constraints to forget. We can, nevertheless, modify the Forget step to forget all constraints and obtain a truly stochastic version of the algorithm. In this version, at each iteration, we choose a random set of constraints and project onto these constraints only, independently of what constraints were used in previous iterations. We cannot, however, forget the values of the dual variables. This version is similar to that in \cite{article}. However, \citet{article} only looks at the problem when we have linear equality constraints. To employ such an approach, we could modify our problem to add slack variables and change all of constraints into equality constraints, however these modifications will not yield an equivalent problem. One of the major assumptions of \cite{article} is that the objective function is strictly convex. Thus, we if add slack variables, then we would need to modify our objective function to be strictly convex on these variables as well. This changes the problem.  

\subsection{Convergence Analysis: Linear Inequality Constraints} 
\label{sec:convergence-linear} 

Before we can use \textsc{Project and Forget}, it is crucial to establish a few theoretical properties. Previous work on the convergence of the Bregman method relies on the fact that the algorithm cyclically visits all of the constraints. For our method, however, this is not the case and so it is not clear that the convergence results for the traditional Bregman method still apply. Fortunately, the proofs for the traditional Bregman method can be adapted in subtle ways, so that we can establish crucial theoretical properties of the \textsc{Project and Forget} algorithm.  

\begin{thm} \label{thm:linear} 
If $f \in \mathcal{B}(S)$, $H_i$ are strongly zone consistent with respect to $f$, and $\exists\, x^0 \in S$ such that $\nabla f(x^0) = 0$, then 
\begin{enumerate} 
\item \label{part:1prime} If the oracle $\mathcal{Q}$ satisfies property \ref{prop:sep} (property \ref{prop:rand}), then any sequence $x^n$ produced by the above algorithm converges (with probability $1$) to the optimal solution of problem \ref{problem:linear}. 
\item  \label{part:3prime} If $x^*$ is the optimal solution, $f$ is twice differentiable at $x^*$, and the Hessian $H := H f(x^*)$ is positive definite, then there exists $\rho \in (0,1)$ such that  
\begin{equation} \label{eq:conv}
\lim_{\nu \to \infty } \frac{\|x^* - x^{\nu+1} \|_H}{\|x^* - x^\nu \|_H} \le \rho  
\end{equation} 
where $\|y\|_H ^2 = y^THy$. In the case when we have an oracle that satisfies property \ref{prop:rand}, the limit in \ref{eq:conv} holds with probability $1$. 
\end{enumerate} 
\end{thm}

\noindent The proof of Theorem \ref{thm:linear} also establishes another important theoretical property.  
\begin{prop} \label{prop:active} 
If $a_i$ is an inactive constraint, then there exists an $N$, such that for all $n \ge N$, we have that $z^n_i=0$. That is, after some finite time, we never project onto inactive constraints ever again. 
\end{prop}
 
\begin{cor} \label{cor:dualconverge}
Under the assumptions for part \ref{part:3prime} of Theorem \ref{thm:linear}, the sequence $z^n \to z^*$ also converges. 
\end{cor} 

These properties are important as they permit the following interpretation of our algorithm. The algorithm spends the initial few iterations identifying the active constraints from amongst a large number of constraints. This is the active set part of the algorithm. The algorithm then spends the remainder of the iterations finding the optimal solution with respect to these constraints. Empirically, we notice this phenomenon as well. At first, the error metrics decreases very slowly, while the number of constraints that are being considered grows rapidly. Eventually, we reach a point when the number of constraints that we are currently considering stabilizes, at this point the error metrics start decreasing very rapidly. An example of this phenomenon can be seen in Figure \ref{fig:error}. This behavior is one of the major advantages of our method. Additionally, the ability to find the set of active constraints without having to solve the problem is another advantage of our algorithm. 

\begin{rem} \label{rem:facet} We note that while these results show that the algorithm converges linearly, $\rho$ is close one. Indeed, the proof bounds  $\rho \le \frac{F}{F+1}$, where $F$ is the number of hyperplanes that the optimal solution lies on. From a heuristic perspective, it is beneficial to use only those hyperplanes that define the facets of the constraint polytope. \end{rem}

For the truly stochastic case, we have the following theorem instead.

\begin{thm} \label{thm:linear-stochastic} 
If $f \in \mathcal{B}(S)$, $H_i$ are strongly zone consistent with respect to $f$, and $\exists\, x^0 \in S$ such that $\nabla f(x^0) = 0$, then with probability $1$ any sequence $x^n$ produced by the above truly stochastic algorithm converges to the optimal solution of problem \ref{problem:linear}. Furthermore, if $x^*$ is the optimal solution, $f$ is twice differentiable at $x^*$, and the Hessian $H := H f(x^*)$ is positive semi-definite, then there exists $\rho \in (0,1)$ such that with probability $1$, 
\begin{equation}
\liminf_{\nu \to \infty } \frac{\|x^* - x^{\nu+1} \|_H}{\|x^* - x^\nu \|_H} \le \rho. \label{eq:conv2}
\end{equation} \end{thm}

Because the proofs of Theorem~\ref{thm:linear} and \ref{thm:linear-stochastic} are quite technical and involve two different types of separation oracles, we split them into several parts. In Subsections~\ref{sec:proof11} and~\ref{sec:proof12}, we prove the first part of Theorem~\ref{thm:linear} for separation oracles with property~\ref{prop:sep} and~\ref{prop:rand}, respectively. In Subsection~\ref{sec:proofpart2}, we prove the second part of Theorem~\ref{thm:linear} (also subdividing this proof into several cases). Finally, in Subsection \ref{sec:proof-linear-stochastic} we prove Theorem \ref{thm:linear-stochastic}.

\section{Project and Forget: General Convex Constraints}

In the previous section, we detailed the \textsc{Project and Forget} algorithm for half-space (or linear inequality) constraints. In this section, we use similar idea to develop the appropriate variations for general convex constraints. We draw inspiration from Dijkstra's method and use a slightly different set of assumptions on our objective functions and constraints. These assumptions are detailed in Section \ref{sec:general-convex-constraints}.

\subsection{Algorithm}

Let $x^n$ be the primal sequence of iterates and $q^n$ an auxiliary sequence. Let $P_k$ denote the Bregman projection operator onto the $k$th constraint set. Let $f^*$ be the convex conjugate of $f$. Let $i(k)$ denote the control sequence, and let $p(c,k) = \argmax_{k' < k} i(k') = c$. We will abbreviate $p(i(k),k)$ as $p(k)$. With this notation established, the \textsc{Project and Forget} algorithm for the case of general convex constraints is shown in Algorithm \ref{alg:bgreman-general}.

\begin{algorithm}[!ht]
\caption{Project and Forget algorithms}
\label{alg:bgreman-general}
\begin{algorithmic}[1]
\State Initialize $q^0  = 0$, $L = \emptyset$
\Function{Project}{$x_0 = x,q,L$}
	\For{$C_i \in L$}
		\State$x^{n} := (P_{i} \circ \nabla f^*)(\nabla f(x^{n-1}) + q^{p(n)})$ \label{line:1}
		\State $q^n = \nabla f(x^{n-1}) + q^{p(n)} - \nabla f(x^n) $ \label{line:2}
	\EndFor
	\Return $x^{|L|}$
\EndFunction
\Function{Forget}{$x,q,L$}
	\For{$C_i \in L$}
		\If{$q^{p(i,n)} == 0$} Forget $C_i$ 
		\EndIf
	\EndFor
	\Return $L$
\EndFunction
	
\end{algorithmic}
\end{algorithm}

Lines \ref{line:1} and \ref{line:2} of the above algorithm come from Dijkstra's method. To understand their role, note that since $f$ is Legendre, we have that $\nabla f^* = (\nabla f)^{-1}$. Thus, on line \ref{line:1}, we are perturbing $x^{n-1}$ by modifying its gradient with the auxiliary variable $q^{p(n)}$. For example, if $f(x) = x^TQ x$ for some positive definite matrix $Q$, then line \ref{line:1} would be 
\[ 
	x^{n} = P_i(Q^{-1}(Qx^{n-1} + q^{p(n)})) = P_i (x^{n-1} + Q^{-1} q^{p(n)}). 
\]

\subsection{Convergence Analysis}

Now that we have defined the above algorithm, the second major theoretical result of this paper is the following. 

\begin{thm} \label{thm:convergence-general} If $f$ is a closed very strictly convex, co-finite, Legendre function and we are given a point $x^0 \in dom f$, then 
\begin{enumerate} 
\item if the oracle $\mathcal{Q}$ satisfies either property \ref{prop:sep} or \ref{prop:rand}, then any sequence $x^n$ produced by the above algorithm converges to Bregman projection of $x^0$ on $C$ with respect to $f$; and,
\item if $f$ is also a quadratic function, then for large enough $\nu$ there exists a $\rho \in (0,1)$ such that  \[ \|x^* - x^{\nu+1} \| \le \rho \|x^* - x^\nu \|. \] 
\end{enumerate} 
\end{thm}

\section{Applications: Metric Constrained Problems} 
\label{sec:applications}

To demonstrate the effectiveness of our method in solving metric constrained problems, we solve large instances of two different types of such problems: metric nearness and correlation clustering. We focus on these types of problems first as they were the original motivation behind our work. \footnote{All implementations and experiments can be found at \url{https://github.com/rsonthal/ProjectAndForget}.} 

\subsection{Metric Nearness} The first and simplest form of a metric constrained problem is the metric nearness problem. Following~\citet{Brickell2008TheMN}, the metric nearness problem is: 
\begin{quote} given a point $x \in \R^{\binom{n}{2}}$, find the closest (in some $\ell_p$ norm) point $x^* \in$ MET$_n$ to $x$. 
\end{quote}
This problem is a form of metric learning; see \citet{Brickell2008TheMN} for an application to clustering and see ~\citet{Gilbert2018UnsupervisedML} for an application to unsupervised metric learning. Additionally, if we further restrict to finding the closest Euclidean metric, then this problem is a well studied one. See  \cite{emn1,emn2}, and \cite{emn3} for examples of this problem. Recently this problem has also been looked at from a discrete setting. \cite{GilbertJain2017} and \cite{Fan2018MetricVD} looked to solve the problem with the fewest possible changes. Following this, \cite{Fan2018GeneralizedMR} generalized the problem to finding the closest point in $\met(G)$ instead of $\met_n$.

We use $\ell_2$ version of this problem to demonstrate that standard solvers have significant drawbacks on large scale metric constrained problems while \textsc{Project and Forget} handles these problems easily. In particular, in addition to comparing against the cyclic Bregman used in \citet{Brickell2008TheMN}, we compare against commercial solvers CPLEX \citep{cplex} and Mosek \citep{mosek}, ADMM based solvers OSQP \citep{osqp}, SCS \citep{scs} and COSMO \citep{cosmo}, operator splitting and interior point solvers Ipopt \citep{ipopt}, ProxSDP \citep{proxsdp} and ECOS \citep{ecos}, and active set solver SLSQP \citep{slsqp}. We also use Mosek as an active set solver (MASS) as follows. We go through all $O(n^3)$ constraints and find the subset $S$ of the violated constraints. We then use Mosek to solve the problem on $S$. We then remove the inactive constraints and add in the new violated constraints and use Mosek to solve the problem again and iterate.

\subsubsection{Experimental Set Up:}

Before we see the experimental results, let us look at the experimental set up more closely. 

\textbf{Data.} For this experiment, we will generate three different types of synthetic data. We will refer to these as Type I, Type II, and Type III data. 

For Type I data, we generate random weighted complete graphs with Gaussian weights.  For Type II data, for each edge $e$ we set $w(e) = 1$ with probability $0.8$ and and $w(e) = 0$ with probability $0.2$. For Type III data, we let $u_{ij}$ be sampled from the uniform distribution on $[0,1]$ and $v_{ij}$ from a standard normal, then the weight for an edge $e=ij$ is given by \[ w_{ij} = \left\lceil 1000\cdot u_{ij} \cdot v_{ij}^2 \right\rceil. \]

\textbf{Implementation Details.} We implemented the algorithm from \cite{Brickell2008TheMN}. We made a small modification that improves the running time. In \cite{Brickell2008TheMN}, it is recommended that we store the dual variable $z$ as a sparse vector. However, as we do not want the overhead of handling sparse vectors, we store $z$ as a dense vector. 

\textsc{Project and Forget}, as presented Algorithm \ref{alg:imp1}, was implemented with two modifications. As we can see from Algorithm \ref{alg:apsp}, when the oracle finds violated constraints, it looks at each edge in $G$ and then decides whether there is a violated inequality with that edge. It is cleaner, in theory, to find all such violated constraints at once and then do the project and forget steps. It is, however, much more efficient in practice to do the project and forget steps for a single constraint as we find it. This approach also helps cut down on memory usage. Once our oracle returns a list of constraints (note we have already projected onto these once), we project onto our whole list of constraints again. Thus, for the constraints returned by the oracle, we project onto these constraints twice per iteration. Note this does not affect the convergence results for the algorithm. 

\begin{algorithm}[!htbp]
\caption{Pseudo-code for the implementation for Metric Nearness.}
\label{alg:imp1}
\begin{algorithmic}[1]
 \State $L^0 = \emptyset$, $z^0 = 0$. Initialize $x^0$ so that $\nabla f(x^0) = 0$.
    \While{Not Converged}
	\State Let $d(i,j)$ be the weight of shortest path between nodes $i$ and $j$ or $\infty$ if none exists. 
	\State $L = \emptyset$
	\For{Edge $e = i(,j) \in E$}
		\If{$w(i,j) > d(i,j)$}
			\State Let $P$ be the shortest path between $i$ and $j$.
			\State Let $C = P \cup \{(i,j)\}$. 
			\State Project onto $C$ and update $x,z$.
			\If{$z_C != 0$}
				\State Add $C$ to $L$.
			\EndIf
		\EndIf
	\EndFor
        \State $\tilde{L}^{\nu+1} = L^\nu \cup L$
        \State $x^{\nu+1}, z^{\nu+1}$ $=$ Project($x^\nu, z^{\nu}, \tilde{L}^{\nu+1}$) 
        \State $L^{\nu+1}$ $=$ Forget($\tilde{L}^{\nu+1}$) 
    \EndWhile
    \Return $x$
\end{algorithmic}
\end{algorithm}

The solvers CPLEX, Mosek, OSQP, SCS, COSMO, Ipopt, ProxSDP, ECOS, and MASS all were accessed via Julia's JuMP library \citep{JuMP}. SLSQP was inferfaced with using Scipy. 

\textbf{Convergence criterion.} When we presented the algorithm in Section \ref{sec:Algorithm-linear}, we did not specify a convergence criterion because we wanted the criterion to be application dependent. The convergence criterion for \textsc{Project and Forget} and \citet{Brickell2008TheMN} is as follows. One variant of the metric nearness problem is the decrease only variant, in which we are not allowed to increase the distances and must only decrease them. This problem can solved in $O(n^3)$ time by calculating the all pairs shortest path metric \citep{GilbertJain2017}. Given $x^n$ as input, let $\hat{x}^n$ be the optimal decrease only metric. For \textsc{Project and Forget}, and \cite{Brickell2008TheMN}, we ran these experiments until $\|\hat{x}^n - x^n\|_2 \le 10^{-10}$. For convenience, given $x^n$ let $D(x^n) = \|\hat{x}^n - x^n\|_2$. Note that $D(x^n)$ is an upper bound for how far the current iterate is from the polytope. We can use the information returned by our oracle to compute $D(x^n)$ in quadratic time, as our oracle computes $\hat{x}^n$. For these experiments, however, at the end of each iteration, we rerun the computation. We do this for fairness, as this extra computation has to be added to Bregman's cyclic method.

For CPLEX. Mosek, OSQP, SCS, COSMO, Ipopt, ProxSDP, ECOS, SLSQP, we let the convergence criterion be the default criterion. For MASS, we ran it until the constraint set converged. 

\textbf{Comparison Statsitics} We compare the solvers on two different kinds of measures. The first is the convergence details of each solver. To determine convergence, we look at two statistics. First, we look at the relative objective difference ($ROD$), which is given by 
\[
	ROD = \frac{\text{OtherSolverObjective - ProjectForgetObjective}}{\text{ProjectForgetObjective}}.
\]  
Second, if $x$ is the solution returned by \textsc{Project and Forget} and $\tilde{x}$ is the solution returned by one of the other solvers, then the feasibility difference ($FD$) is given by 
\[
	FD = D(\tilde{x}) - D(x).
\]
For both metrics, if the quantities are positive, then \textsc{Project and Forget} does better.

The second measure we compare is the time taken to solve the optimization problem. \emph{Here we use the solve time reported by the solvers. Note that this does not include the interface time and in many cases the interface time could be significant}.

\subsubsection{Results}

As we can see from Table \ref{table:mn}, standard solvers run out memory or start taking too long extremely rapidly as a function of instance size. In fact, all times reported for the standard solvers is the solve time returned by the optimizer and does not include the interface time. In many cases, such as ProxSDP, the interface time could be multiple hours. On the other hand, while initially \cite{Brickell2008TheMN} is faster than \textsc{Project and Forget} as $n$ gets larger, \textsc{Project and Forget} starts to dominate. Thus, showing that \textsc{Project and Forget} is the only viable algorithm to solve large metric constrained problems. 

\begin{table*}[!htbp]
\centering
\begin{tabular}{c|cccccccccc}
\toprule
Algorithm &  \multicolumn{10}{c}{Number of Nodes}  \\
 & 100 & 200 & 300 & 400 & 500 & 600 & 700 & 800 & 900 & 1000 \\
\midrule
Ours & 13.5 & 32.7 & 85.1 & 170 & \cellhi 271 & \cellhi 458 & \cellhi 720 & \cellhi 983 &\cellhi 1356 & \cellhi 1649  \\
Cyclic Bregman & \cellhi 1.77 & \cellhi 10.5 &\cellhi 47.1 &\cellhi 141 & 322 & 558 & 910 & 1472 & 2251 & 3167 \\
Mosek & 11.7 & 542 & \multicolumn{8}{l}{Out of Memory} \\
SCS & 1632 & 19466 & \multicolumn{8}{l}{Timed Out} \\
OSQP & 64.5 & 3383 & \multicolumn{8}{l}{Timed Out}  \\
ProxSDP & 353 & 684 & \multicolumn{8}{l}{Timed Out}  \\
Ipopt & 2792 &  \multicolumn{9}{l}{Timed Out}  \\
ECOS & 597 & \multicolumn{9}{l}{Timed Out} \\
CPLEX &  \multicolumn{10}{l}{Out of Memory}  \\
SLSQP &  \multicolumn{10}{l}{Timed Out}  \\
COSMO &  \multicolumn{10}{l}{Timed Out}  \\ 
\bottomrule
\end{tabular}
\caption{Table comparing \textsc{Project and Forget} against a variety of different solvers to solve the Metric Nearness problem for Type 1 graphs in terms of time taken in seconds. All experiments were run on a Computer with 52 GB of memory. All times reported are averaged over 5 instances.}
\label{table:mn}
\end{table*}

To make sure that we are running all of the algorithms to the same level of convergence, we take the five best solvers and check their convergence statistics. That is, we check the convergence statistics for commercial solver Mosek, for ADMM based solvers OSQP and SCS, and for operator splitting solver ProxSDP.  

Each data point in Table \ref{table:convergence} is average over 10 trials. When comparing against \cite{Brickell2008TheMN}, we let $n$ range from a 100 to 1000, i.e., the same values as in Table \ref{table:mn}. Table \ref{table:convergence} then shows that both method have similar levels of convergence. Thus, since \textsc{Project and forget} is faster; it is the superior algorithm. 

To compare against the rest of the solvers, as they simply cannot solve the problem for larger values of $n$, we let $n$ range from $10$ to $100$. Here, we see that \textsc{Project and Forget} has much smaller feasibility error than the other solvers. Solvers such as \textsc{SCS} and \textsc{QSQP} have a feasibility error of about $10^{-3}$ to $10^{-5}$, however, we run \textsc{Project and Forget} until the feasibilty error is smaller than $10^{-10}$, which is many orders of magnitude smaller. For ProxSDP, $ROD$ was consistently around 1 and $FD$ was consistently greater than $2$. We conclude that these solvers have not converged. 

The only solver, other than \textsc{Project and Forget}, that consistently converges is Mosek, which has higher objective values than \textsc{Project and Forget} in all but one case. The active set method MASS was run until the the feasibility error was at most 1e-10 or the constraint set was stable, in practice we found that the constraint set stabilized first and hence the solver has similar convergence statistics to Mosek. 

\begin{table*}
\setlength{\tabcolsep}{1pt} 
    \centering
    \begin{tabular}{cc|cccccccccc}
    \toprule
     & & \multicolumn{10}{c}{$n$} \\
    Solver & & $100$ & $200$ & $300$ & $400$ & $500$ & $600$ & $700$ & $800$ & $900$ & $1000$ \\
    \midrule
   Cyclic Bregman & $ROD$ & -3e-13 & -4e-14 & 4e-15 & -2e-14 & 6e-15 & -8e-16 & 1e-15 & -4e-15 & 5e-15 & 2e-17\\
    & $FD$ & -6e-12 & 5e-13 & 4e-13 & -3e-12 & 1e-12 & 3e-12 & 5e-12 & 3e-12 & -2e-12 & 7e-12   \\
    \midrule
    & & \multicolumn{10}{c}{n} \\
     & & $10$ & $20$ & $30$ & $40$ & $50$ & $60$ & $70$ & $80$ & $90$ & $100$ \\
    \midrule
       Mosek & $ROD$ & -4e-4 & 2e-9 & 1e-9 & 6e-9 & 7e-9 & 1e-8 & 2e-8 & 2e-8 & 2e-8 & 2e-8\\
        & $FD$ & 2e-11 & 2e-11 & 1e-9 & 7e-10 & 3e-10 & 9e-10 & 1e-9 & 9e-10 & 8e-10 & 7e-10\\
        \midrule
       OSQP & $ROD$ & -4e-4 & 5e-7 &-9e-7 & -2e-7 & -7e-8 & 1e-7 & -1e-8 & 7e-8 & 4e-8 & 5e-8\\
        & $FD$ & 1e-3 & 2e-3 & 1e-3 & 9e-4 & 8e-4 & 5e-4 & 9e-4 & 1e-3 & 8e-4 & 1e-3\\
        \midrule
         SCS & $ROD$ & -4e-4 & -3e-8 & 8e-7 & -8e-6 & -5e-7 & -1e-8 & -2e-8 & 3e-9 & -7e-6 & 1e-9\\
        & $FD$ & 7e-5 & 1e-4 & 5e-3 & 6e-3 & 5e-4 & 9e-5 & 4e-5 & 4e-5 & 0.09 & 4e-5\\
        \bottomrule
    \end{tabular}
    \caption{Convergence statistics for the metric nearness problem for the different solvers. Each value is the average over 10 trials.}
              \label{table:convergence}
\end{table*}

As we can see, the cyclic Bregman method is the one that has the closest performance to \textsc{Project and Forget}. Therefore, we want to take a closer look at how ROD and FD evolve for the two algorithms. Figure \ref{fig:mn-internal}, shows a plot for both of these statistics as a function of number of projections. Note we did not compute these quantities after every projection but only after each iteration of the algorithm. As we can see from Figure \ref{fig:mn-internal}, \textsc{Project and Forget} uses far fewer projections. In fact, \textsc{Project and Forget} uses fewer projections than the cyclic Bregman method does in one iteration. Thus, highlighting the great advantage obtained by the forget step.  

Figure \ref{fig:mn-internal} also highlights another aspect of the algorithm from its analysis in Proposition \ref{prop:active}. That is, the algorithm spends the first few iterations finding the correct set of constraints. Once it has found the correct set of constraints, it converges to the correct solution very quickly (with an exponential decay in the error). This is seen in Figure \ref{fig:mn-internal}, where sometime between $10^6$ and $10^7$ projections, the algorithm seems to have found the correct active set of solutions.

\begin{figure}
    \centering
    \subfigure[Plot for Relative Objective Difference (ROD) versus number of projections for \textsc{Project and Forget} and Bregman's cyclic method]{\includegraphics[width = 0.49\linewidth]{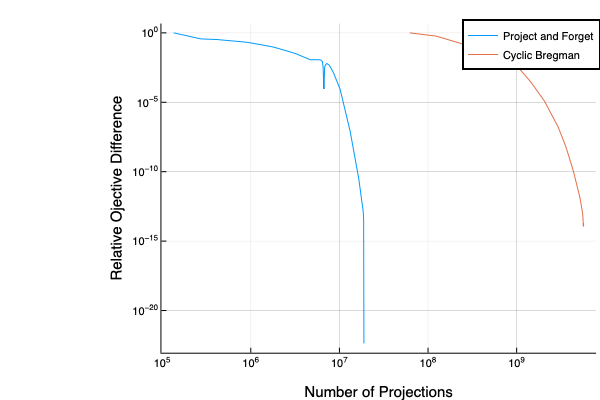}}
    \subfigure[Plot for Feasibility Difference (FD) versus number of projections for \textsc{Project and Forget} and Bregman's cyclic method]{\includegraphics[width = 0.49\linewidth]{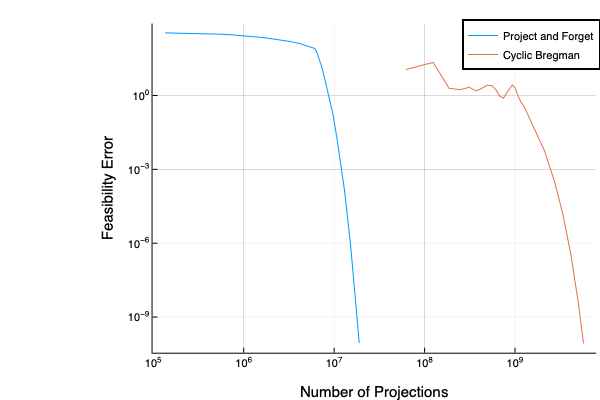}}
    \caption{Plot for ROD and FD versus number of projections for \textsc{Project and Forget} and Bregman's cyclic method. }
    \label{fig:mn-internal}
\end{figure}

We further tested our method against the cyclic Bregman method on Type II and Type III data. In this case, we got the running times shown in Figure \ref{fig:metricnearnessappendix}. For this we relaxed the convergence criteria to being within 1 of the closest decrease only metric solution. One thing we learn from relaxing the convergence criteria is that cyclic Bregman method outperforms \textsc{Project and Forget} for larger values of $n$. This is because, that \textsc{Project and Forget} only focuses on the set of the active constraints. Thus, the more stringent the convergence criterion, the better our method does compared to the standard cyclic Bregman method used in \cite{Brickell2008TheMN}. 

\begin{figure}[!htbp]
\centering\hfill
\subfigure[Type one graphs]{\includegraphics[width=0.3\linewidth]{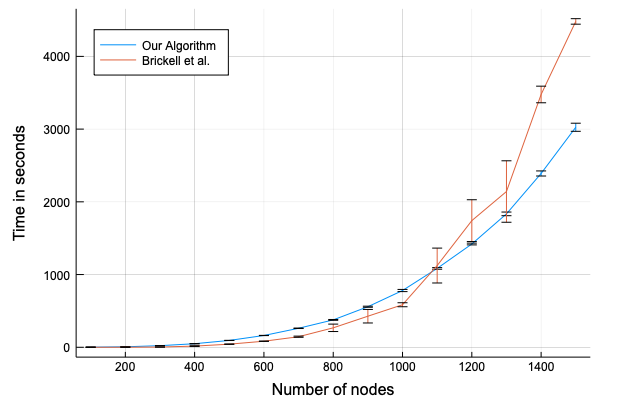}}\hfill
\subfigure[Type two graphs]{\includegraphics[width=0.3\linewidth]{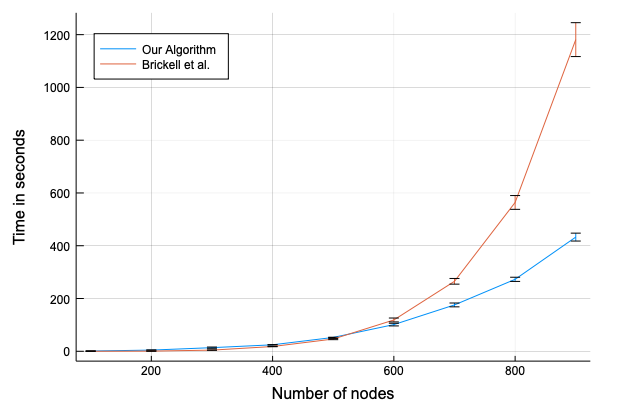}}\hfill
\subfigure[Type three graphs]{\includegraphics[width=0.3\linewidth]{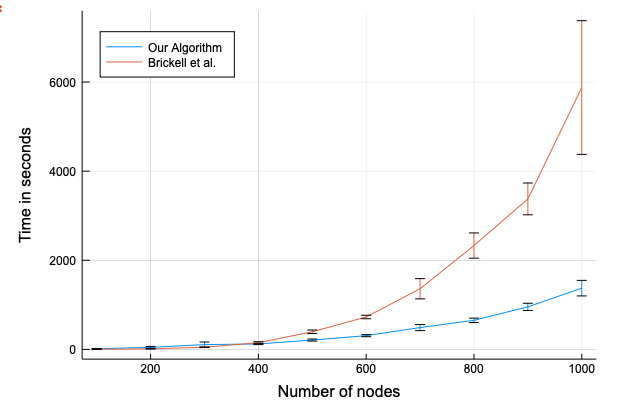}}\hfill
\caption{The red line is the mean running time for the algorithm from \cite{Brickell2008TheMN}. The blue line is the running mean time for our algorithm. All computations were done on a machine with 4 physical cores, each with 13 GB of RAM. }
\label{fig:metricnearnessappendix}
\end{figure}


Finally, we must note that when we use the cycle constraints to represent $\met(K_n)$, there is significant dependency amongst the constraints. That is, many constraints are not independent from other constraints. We may suspect that our oracle finds a lot of active constraints that have linear dependencies and so $L^{(\nu)}$ becomes saturated with dependent constraints. It is surprising, however, that after we solve the problem, $L^{(\nu)}$, more often than not, only contains 3-cycle constraints and does not have a lot of dependency. We also see that the size of $L^{(\nu)}$ is about $O(n^2)$.

\subsection{Weighted Correlation Clustering on General Graphs.} 

For the correlation clustering problem, we are given a graph $G = (V,E)$ (not necessarily complete) and two non-negative numbers $w^+(e)$ and $w^-(e)$ for each edge $e$. These numbers indicate the level of similarity and dissimilarity between the end points of $e$. The goal of correlation clustering is to partition the nodes into clusters so as to minimize some objective function. The most common objective is $ \sum_{e\in E}w^+(e)x_e + w^-(e) (1-x_e) $, where $x_e \in \{0,1\}$ indicates whether the end points of the edge $e$ belong to different clusters. This problem is NP-hard and many different approximation algorithms and heuristics have been developed to solve it. The best approximation results \citep{corOpt1, corOpt2}, however, are obtained by rounding the solution to the following relaxed linear problem 
\begin{equation} 
\begin{array}{ll@{}ll}
\label{problem:CC}
\text{minimize}  &  \sum_{e\in E}w^+(e)x_e + w^-(e) (1-x_e) & \\
\text{subject to}& x_{ij} \le x_{ik} + x_{kj} \ \ \ \ \  \ \  i,j,k=1 ,..., n\\
                 &                                                x_{ij} \in [0,1] \ \ \ \ \ \ \ \ \ \ \  \ \ i,j=1 ,..., n.
\end{array}
\end{equation}
Special cases, such as when the weights are $\pm1$ and $G = K_n$, have faster algorithms for the same approximation ratio \citep{Ailon2008AggregatingII}.

The LP formulation for correlation clustering in Equation~\ref{problem:CC} has $\Theta(n^3)$ constraints. Hence, solving the LP for large $n$ becomes infeasible quickly in terms of both memory and time. \citet{veldt} showed that for instances with $n \approx 4000$, standard solvers such as Gurobi ran out of memory on a 100 GB machine. On the other hand, \citet{veldt} developed a method using which they can feasibly solve the problem for $n \approx 11000$. 
To do so, they transform Problem \ref{problem:CC} into Problem \ref{problem:CCtransformed}. To do this transformation, for $e\in E$, we define $ \tilde{w}(e) = |w^+(e) - w^{-}(e)|$. For $e \not\in E$, we let $\tilde{w}(e) = 0$ Then $W$ is a diagonal matrix whose entries are given by $\tilde{w}$. Then, we define $d$ as follows 
\[ 
    d_{e} = \begin{cases} 1 & w^{-1}(e) > w^+(e) \\ 0 & \text{otherwise} \end{cases}.
\]
And the transformed problem is as follows. 
\begin{equation}
\label{problem:CCtransformed}
\begin{array}{ll@{}ll}
\text{minimize}  &  \tilde{w}^T |x- d| + \frac{1}{\gamma} |x - d|^T W |x-d|. & \\
\text{subject to}& x \in \text{MET}(K_n). \\
\end{array}
\end{equation}
For general $\gamma$, the solution to Problem \ref{problem:CCtransformed} approximates the optimal solution to \ref{problem:CC}. However, for large enough $\gamma$ it has been shown that the two problems are equivalent \citep{veldt}. As we will see for our instances, Problem \ref{problem:CCtransformed} is at most a $2$ approximation to problem \ref{problem:CC}, which is an $O(\log(n))$ approximation of the NP-hard correlation clustering problem.

Finally, we also relax the condition that $x \in \text{MET}_n$ to $x \in \text{MET}(G)$. The formulation of the LP that we solve is as follows:

\begin{equation}
\begin{array}{ll@{}ll}
\text{minimize}  & \tilde{w}^T f + \frac{1}{\gamma} f^T \cdot W \cdot f& \\
\text{subject to}& x \in \text{MET}(G) \\
& f_{ij} = |x_{ij} - d_{ij}|, \ \ \ & (i,j) \in E. \\
\end{array}
\label{problem:final}
\end{equation}

\noindent While it seems like we have relaxed the problem, however, since our objective function does not depend on the value of $x$ for the edges not in $G$ we see that the two problems are equivalent due to Proposition \ref{prop:fiber}.  

\begin{rem} When \citet{veldt} tested their solver against Gurobi, they did so in an active set manner. That is, they found a set of violated constraints, fed this set into Gurobi, and solved the problem with this subset of the constraints. They then took the solution from Gurobi and found constraints that the current solution violated and added these constraints and repeated. They did this until Gurobi solved the problem. We view this convoluted process as yet more evidence for standard active set methods not being feasible at large scales.
\end{rem}

\subsubsection{Experimental Set up}

Before looking at the results, let us see the experimental set up. 

\textbf{Data.} We solve the problem for two different types of graphs; dense and sparse. 

For dense graphs, we use four graphs from the Stanford sparse network repository. Then, following \citet{veldt}, we use the method from \citet{Wang2013ASA} to convert these graphs into instances of weighted correlation clustering on the complete graph. We compare our method against \citet{veldt2}, a parallel version of \cite{veldt}, in terms of running time, quality of the solutions, and memory usage. 

For much real-world data, the graph $G$ is larger than our previous experiments but is also sparse. Since the weighting of the edges does not affect the size of the linear program that needs to be solved, we tested our algorithm on sparse signed graphs to get an estimate of the running time for the algorithm. The two graphs used for this experiment are much bigger instances than our previous experiments and have 82,140 nodes and 131,828 nodes, respectively. 

\textbf{Implementation Details.} For the case when $G=K_n$, in addition to the modifications that were done for metric nearness experiment, we made two more modifications to \textsc{Project and Forget}. First, we did the project step and the forget step one additional time per iteration. Second, we parallelized the oracle by running Dijkstra's algorithm in parallel. The pseudo-code for this version of Algorithm \ref{alg:bregman} can be seen in Algorithm \ref{alg:imp2}. For the sparse version, we also used the parallel version of the oracle and during each iteration, but we did the project and the forget step 75 times per iteration.

\begin{algorithm}[!htbp]
\caption{Pseudo-code for the implementation for CC for the dense case.}
\label{alg:imp2}
\begin{algorithmic}[1]
 \State $L^0 = \emptyset$, $z^0 = 0$. Initialize $x^0$ so that $\nabla f(x^0) = 0$. 
 \State $L_a$ is the list of additional constraints. $z_a^0 = 0$ (dual for additional constraints)
    \While{Not Converged}
	\State Let $d(i,j)$ be the weight of shortest path between nodes $i$ and $j$ or $\infty$ if none exists. This is found using a parallel algorithm. 
	\State $L = \emptyset$
	\For{Edge $e = i(,j) \in E$}
		\If{$w(i,j) > d(i,j)$}
			\State Let $C = P \cup \{(i,j)\}$. Where $P$ be the shortest path between $i$ and $j$.
			\State Project onto $C$ and update $x,z$.
			\If{$z_C != 0$}  
			    \State $C$ to $L$.
			\EndIf
		\EndIf
	\EndFor
        \State $L^{\nu} \leftarrow L^\nu \cup L$
        \For{$i = 1,2$}
        		\State $x^{\nu}, z^\nu$ $\leftarrow$ Project($x^\nu, z^\nu, L^{\nu}$) 
		\State $L^{\nu}$ $\leftarrow$ Forget($L^{\nu}$) 
        \EndFor
        \State $x^{\nu}, z_a^\nu$ $\leftarrow$ Project($x^\nu, z_a^\nu, L_a$) 
        \State $x^{\nu+1} = x^{\nu}, L^{\nu+1} = L^\nu$, $z^{\nu+1} = z^\nu$, $z^{\nu+1}_a = z^\nu_a$, 
    \EndWhile
    \Return $x$
\end{algorithmic}
\end{algorithm}

Note that for both experiments, the additional constraints that were introduced due to the transformation were all projected onto once per iteration and never forgotten. The pseudo-code for this version can be seen in Algorithm \ref{alg:imp3}. We used the implementation provided by the authors of \cite{veldt} to run the experiments for their algorithm.

\begin{algorithm}[!htbp]
\caption{Pseudo-code for the implementation for CC for the sparse case.}
\label{alg:imp3}
\begin{algorithmic}[1]
 \State $L^0 = \emptyset$, $z^0 = 0$. Initialize $x^0$ so that $\nabla f(x^0) = 0$. 
 \State $L_a$ is the list of additional constraints. $z_a^0 = 0$ (dual for additional constraints)
    \While{Not Converged}
	\State Let $d(i,j)$ be the weight of shortest path between nodes $i$ and $j$ or $\infty$ if none exists. This is found using a parallel algorithm. 
	\State $L = \emptyset$
	\For{Edge $e = i(,j) \in E$}
		\If{$w(i,j) > d(i,j)$}
			\State Let $C = P \cup \{(i,j)\}$. Where $P$ be the shortest path between $i$ and $j$.
			\State Add $C$ to $L$.
		\EndIf
	\EndFor
        \State $L^{\nu} \leftarrow L^\nu \cup L$
        \For{$i = 1,\hdots, 75$}
        		\State $x^{\nu}, z^\nu$ $\leftarrow$ Project($x^\nu, z^\nu, L^{\nu}$) 
		        \State $L^{\nu}$ $\leftarrow$ Forget($L^{\nu}$) 
        \EndFor
        \State $x^{\nu}, z_a^\nu$ $\leftarrow$ Project($x^\nu, z_a^\nu, L_a$) 
        \State $x^{\nu+1} = x^{\nu}, L^{\nu+1} = L^\nu$, $z^{\nu+1} = z^\nu$, $z^{\nu+1}_a = z^\nu_a$, 
    \EndWhile
    \Return $x$
\end{algorithmic}
\end{algorithm}

\textbf{Calculating the approximation ratio.} Let $\hat{x}$ be the optimal solution to \ref{problem:CCtransformed}, then if we let $\displaystyle R = \frac{\hat{x}^T W \hat{x}}{2 \gamma \tilde{w}^T \hat{x}}$, by \cite{veldt}, we have that $\hat{x}$ is an $\displaystyle \frac{1+\gamma}{1+R}$ approximation to the optimal solution of \ref{problem:CC}. This is the formula we used to calculate the approximation ratios reported in Tables \ref{table:CC1} and \ref{table:large}. For our experiments we used $\gamma = 1$. Note that since the entries of $\hat{x}$ are non negative, and the entries of $W$ and $\tilde{w}$ are non-negative, we have that $R > 0$. Thus, our approximation ratio is at most 2.

\textbf{Convergence criterion.} We ran the experiment until the maximum violation of a metric inequality was at most 0.01. However, the two algorithms, \cite{veldt2} and ours, have different metric constraints. Specifically \cite{veldt2} only uses all constraints that come from 3 cycles, whereas we use all cycle constraints. Theoretically both sets of constraints define the same polytope, but practically there is a difference. Thus, in practice our algorithm was run to a slightly greater level of convergence than the one from \cite{veldt2}. Additionally, \cite{veldt2} also check a second criteria for convergence, however, in all cases, that criterion was satisfied much before the maximum violation of a metric inequality was at most 0.01.

\subsubsection{Results}

\begin{table*}[!htb]
\setlength{\tabcolsep}{4.0pt} 
\centering
\begin{tabular}{*8c}
\toprule
\multicolumn{2}{c}{Graph} &  \multicolumn{2}{c}{Time (s)} & \multicolumn{2}{c}{Opt Ratio} & \multicolumn{2}{c}{Avg. mem. / iter. (GiB)}  \\
\midrule
{} & n   & Ours   & \citeauthor{veldt2}    & Ours   & \citeauthor{veldt2} & Ours   & \citeauthor{veldt2} \\
CAGrQc & 4158 & \cellhi 2098 & 5577 & \cellhi 1.33 &  1.38 & 4.4& \cellhi 1.3 \\ 
Power & 4941 & \cellhi 1393 & 6082& \cellhi 1.33  &  1.37 & 5.9&\cellhi 2   \\ 
CAHepTh &8638 &  \cellhi 9660& 35021& \cellhi 1.33 &  1.36 & 24&\cellhi 8  \\ 
CAHepPh & 11204& \cellhi 71071 & 135568  & \cellhi 1.33 & 1.46  & 27.5&\cellhi 15 \\ 
\bottomrule
\end{tabular}
\caption{Table comparing \textsc{Project and Forget} against \citet{veldt2} in terms of time taken, quality of solution, and average memory usage when solving the weighted correlation clustering problem on dense graphs.}
\label{table:CC1}
\end{table*}

For the case of dense graphs, we see from Table \ref{table:CC1} that our algorithm takes less time to obtain a better approximation ratio, but requires more memory per iteration. Thus, demonstrating the superiority of our method in terms of CPU time.  Our algorithm requires more memory because the initial few iterations find a large number of constraints. Later, the algorithm forgets these constraints until the number of constraints stabilizes at a reasonable level. Hence, our initial memory usage is much larger than our later memory usage. To see how the number of constraints found by the oracle evolves, we plot the number of constraints found by the oracle and the number of constraints after the forget step in Figure \ref{fig:constraints}. As we can see, after the initial few iterations, the number constraints found sharply reduces and has found the true set of active constraints by the 15th iteration.  Figure \ref{fig:error} also shows us, as expected, \emph{the exponential decay of the maximum violation of a metric constraint.}

Figure \ref{fig:error}, also highlights another aspect our algorithm. That is, for the initial few iterations the error statistics do not decrease. In fact, we have experimentally seen that the error statistics may actually increase for the first few iterations. Proposition \ref{prop:active}, which tells us that our algorithm spends the initial few iterations finding the active constraint set, explains this phenomenon. During this time, the algorithm makes minimal progress towards reducing the error statistics. However, once we have found the active constraints, Theorem \ref{thm:linearIter} is now applicable and we have an exponential decay of the error statistics. Though in contrast to the theory result, the base of the exponent, is smaller than theoretically predicted. 

\begin{figure}[!htb]
\centering
\hfill
\subfigure[Number of constraints.]{\label{fig:constraints}\includegraphics[width=0.49\linewidth]{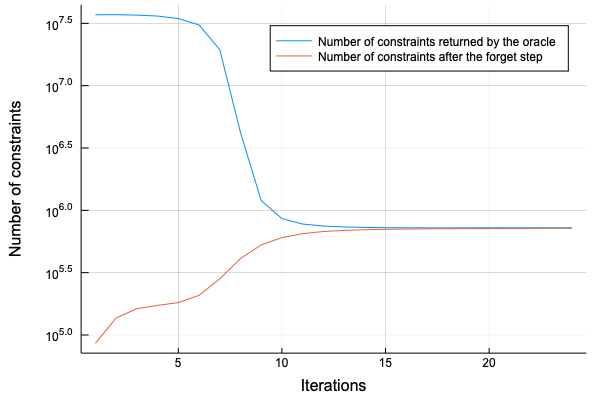}} \hfill
\subfigure[Max Violation.]{\label{fig:error}\includegraphics[width=0.49\linewidth]{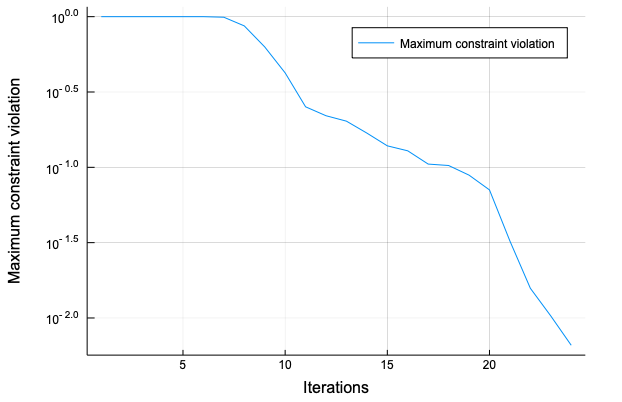}}\hfill
\caption{Plots showing the number of constraints returned by the oracle, the number of constraints after the forget step, and the maximum violation of a metric constraint when solving correlation clustering on the Ca-HepTh graph}
\end{figure}

\begin{rem}
Figure \ref{fig:constraints} highlights the crucial difference between our method and standard active set methods. Standard active set methods would have to initially solve the convex optimization problem with $10^8$ constraints. This would take a very long time. However, using \textsc{Project and Forget}, we only need to compute projections onto each constraint once before we can forget constraints. Thus, we forget constraints more frequently and much earlier. 
\end{rem}

\begin{table*}[!htb]
\setlength{\tabcolsep}{4.3pt} 
\centering
\begin{tabular}{*7c}
\toprule
Graph & $n$ & \# Constraints & Time & Opt Ratio & \# Active Constraints & Iters.\\ \midrule
Slashdot& 82140 & $5.54 \times 10^{14}$ & 46.7 hours &  1.78 & 384227 & 145    \\ 
Epinions &131,828 & $2.29 \times 10^{15}$ &  121.2 hours  & 1.77 &  579926   &   193   \\ \bottomrule
\end{tabular}
\caption{Time taken and quality of solution returned by \textsc{Project and Forget} when solving the weighted correlation clustering problem for sparse graphs. The table also displays the number of constraints the traditional LP formulation would have. }
\label{table:large}
\end{table*}

For sparse graphs, even if we use \citet{veldt2}, based on the average time it took for a single iteration for the CA-HepPh graph, it would take \citet{veldt2} an estimated two days to complete a single iteration, for a graph with $n \approx 80,000$. This is because each of their iterations takes $\Theta(n^3)$ time.  Since most graphs require at least 100 iterations, \citet{veldt,veldt2} cannot be used to solve problems of this magnitude. Other methods of solving the LP are also not feasible as they run out of memory on much smaller instances. 

We can see the performance of our method \textsc{Project and Forget} in Table \ref{table:large}. Here, we see that our method has solved these problems in a reasonable amount of time. As we can see from Table \ref{table:large}, these instances have over 500 trillion constraints, but the number of active constraints is only a tiny fraction of the total number of constraints. Thus, using our approach, we can solve the weighted correlation clustering problem on much larger graphs than ever before. 

There are two reasons as to why \textsc{Project and forget} can be used to solve these problems at these scales. First, while these instances have over 500 trillion constraints, the number of active constraints is only a tiny fraction of the total number of constraints. Second, due to the sparsity of the graph, our oracle finds violated cycle inequalities relatively quickly (sub-cubic time) and, since we forget inactive constraints, we project onto this relatively small number of constraints in each iteration. Thus, using our approach, we can solve the weighted correlation clustering problem on much larger (10 times bigger) graphs than ever before.

\section{Applications: General algorithm}

We shall now also provide experimental evidence that supports the generality of the algorithm. To do so we will solve three different problems, each with general linear constraints (not simply metric constraints) and a variety of objective functions. 

First, we solve the dual of the quadratically regularized optimal transport problem from \cite{blondel18a}. Solving this problem helps highlight various different aspects of our algorithm that are not highlighted by the metric constrained problem. Second, we solve the information theoretic metric learning from \cite{itml}. Solving this problem highlights the wide variety of objective functions that our method can handle. Finally, we use our algorithm to train $L_2$ SVMs to highlight cases when the truly stochastic variant of our algorithm would be useful. 

\subsection{Sparse Optimal Transport}

Optimal transport is a ubiquitous problem that appears many different areas of machine learning, data science, statistics, mathematics, physics, and finance. 

\begin{prob} \label{prob:mkopt} Given two measure spaces $(\gX, \Sigma_X, \mu)$ and $(\gY, \Sigma_Y, \nu)$, and a cost function $c: \gX \times \gY \to \R^+$ , the Monge-Kanterovich Optimal Transport seeks joint probability $\pi$ on $\gX \times \gY$ that minimizes 
\[
    \int_{\gX \times \gY} c(x,y) d\pi(x,y), 
\]
subject to constraint that the marginals of $\pi$ are equal to $\mu$ and $\nu$. 
\end{prob}
In the discrete setting this problem can be formulated as follows:
\begin{equation}
    \begin{array}{ll@{}ll}
    OT(\va,\vb) = \min \langle \mC, \mP \rangle \\
    \text{Subject to: }\ \va = \mP\vone_m, \, \vb = \mP^T\vone_n
    \end{array}
\end{equation}
Where $\mC$ is the cost matrix and $\mP$ is the transport matrix.
It has been shown that optimal solutions to the discrete Monge-Kanterovich problem are sparse \citep{brualdi_2006}. Specifically, at most $n+m-1$ entries of $P$ can be non-zero. 

While Problem \ref{prob:mkopt} can be formulated as a linear program, solving it is fairly challenging as it has a quadratic number of variables and so many alternative formulations have been proposed. One such alternative, called ROT, is presented in \cite{blondel18a}; it uses a different regularizer, most importantly an $L_2$ regularizer. In this case, they show experimentally that the solutions are still sparse. We show that \textsc{Project and Forget} can also be used to solve ROT. 

\begin{equation}
\begin{array}{ll@{}ll}
\label{problem:rot}
\text{Minimize:}  & \langle \mC, \mP \rangle + \gamma \|\va - \mP \vone_m\|^2 + \gamma\|\vb - \mP^T \vone_n\|^2 \\
\text{Subject to:} & \forall i \in [n], \forall j \in [m],\ \  \mP_{ij} \ge 0
\end{array}
\end{equation}

\subsubsection{Dual Problem}

In particular, we will not solve ROT directly; instead, we will solve the dual formulation of ROT. 

\begin{equation} 
\begin{array}{ll@{}ll}
\label{problem:drot}
&\text{Min:}  & \frac{1}{\gamma}\|\vf\|^2 + \frac{1}{\gamma}\|\vg\|^2 -  \langle \vf, \va \rangle - \langle \vg, \vb \rangle \\
&\text{Subject to: } & \vf_i + \vg_j \le \mC_{ij} &
\end{array}
\end{equation}

Generally, solving the dual problem, while it gives us the optimal value of the primal objective function, does not give you the optimal value of the primal variables directly. Thus, going from the solution to Problem \ref{problem:drot} to that of Problem \ref{problem:rot} is non trivial. This brings us to one aspect of \textsc{Project and Forget} that is not highlighted by the metric constrained problems. That is, \begin{quote} if we use \textsc{Project and Forget} to solve \textbf{either} the primal problem or the dual problem, then it finds the optimal solution to \textbf{both} problems.  \end{quote}

This result holds as a result of Proposition \ref{prop:active}. This proposition states that the variables $z$ converge. Since we maintain the KKT conditions throughout the algorithm, the value that $z$ converges to is the optimal solution of the dual problem and, thus, we solve the dual problem. 

\subsubsection{Sparsity}

We want to solve the dual problem instead of the primal because \textsc{Project and Forget} exploits sparsity to its great advantage. 
Specifically, the dual constraints are extremely sparse; each dual constraint involves two variables so we have rapid computation of the projection.

\subsubsection{Experimental Results}

In this subsection, we detail the experimental set up and present the results. 

\textbf{Data.} The experimental set up is as follows. We take two shifted Gaussian distributions with means $\pm 15$ and variance $10$. Then, we  split the interval $[-20,20]$ into $n$ points and create two discrete distributions by sampling the Gaussians on those $n$ points. We use the squared Euclidean distances between the points as the cost function. We set the regularization parameter $1/\gamma = 10^{-3}$. This set up is a very basic example of the optimal transport problem and is an important test example for algorithmic comparisons. 

\textbf{Convergence Details.} The feasibility error for Mosek and CPLEX are those reported by the solvers. For \textsc{Project and Forget}, we calculate the feasibility error by 
\[  \max_{i,j} \vf_i + \vg_j - \mC_{ij}.
\]

\textbf{Solvers.} For the oracle for the \textsc{Project and Forget}, we simply look through all the constraints and the return all of the violated constraints. This oracle clearly satisfies Property \ref{prop:sep}. We solve the dual version of the problem using \textsc{Project and Forget} (PF), Mosek, and CPLEX and the primal version of the problem using Mosek, CPLEX, LBFGSB, and projected gradient descent (PGD).

\textbf{Results.} Table \ref{table:opt} shows that \textsc{Project and Forget} can be used to solve the problem for much larger values of $n$. Additionally, for the smaller values of $n$, we see that \textsc{Project and Forget} is competitive in terms of solve time. One thing to note from Table \ref{table:conv} is that the difference between the Primal objective and the Dual objective for \textsc{Project and Forget} is smaller than $10^{-7}$; however, for Mosek, and CPLEX this is not the case. Thus, while both have similar levels of feasibility error, we see that our solver is still more ``converged'' since the primal dual gap is smaller.

\begin{table}[!htb]
\setlength{\tabcolsep}{6pt} 
\centering
\begin{tabular}{c|ccccc}
\toprule
Algorithm & $501$ & $1001$ & $5001$ & $10001$  & 20001 \\ \midrule
Poject and Forget & 12 &  \cellhi 151 &  1972 &  \cellhi 5909 &  \cellhi 21665 \\
Cyclic Bregman & 2681 & - & - & - & - \\
LBFGSB & 24 & 162 & 4080 & \multicolumn{2}{l}{Out of memory.} \\
Mosek dual & 56 & 328 & \cellhi 1927 & \multicolumn{2}{l}{Out of memory.} \\
Mosek primal &  \cellhi 5 & \multicolumn{4}{l}{Out of memory.} \\
CPLEX primal & 105 & \multicolumn{4}{l}{Out of memory.} \\
CPLEX dual & \multicolumn{4}{l}{Out of memory.} \\
PGD & \multicolumn{4}{l}{Did not converge.}\\
\bottomrule
\end{tabular}
\caption{Time taken in seconds to solve the quadratic regularized optimal transport problem. All experiments were run on a machine with 52 GB of RAM.}
\label{table:opt}
\end{table}
 
 \begin{table}[p]
     \centering
     \begin{tabular}{c|cccc}
     \toprule
      & \multicolumn{2}{c}{Objective} & \multicolumn{2}{c}{Feasibility Error} \\ \midrule
     Solver  & Dual & Primal & Dual & Primal \\ \midrule
     & \multicolumn{4}{c}{$n=501$} \\ \midrule
        Project and Forget  & 3.8416077 & 3.8416077 & \cellhi 1.7e-09 & \cellhi 0 \\
        Mosek Dual   & 3.8414023 & 3.8414023 & 3.8e-08 & 2.8e-10  \\
        LBFGSB  & n/a & 3.8416114 & n/a & \cellhi 0\\
        Mosek Primal  &  3.8303160 & 3.8303203  & 3.5e-8 & 2.2e-11 \\
        CPLEX Primal  & 3.8416376 &  3.8416076 & 8.44e-07 & 1.13e-04 \\
        CPLEX Dual  & \multicolumn{4}{c}{Ran out of memory} \\ \midrule
        & \multicolumn{4}{c}{$n=1001$} \\ \midrule
        Project and Forget & 1.947532046 & 1.947532046 & \cellhi 2.0e-8 &  \cellhi 0 \\
        Mosek Dual   & 1.947091229 & 1.947091203 & 2.7e-08 & 8.4e-11 \\
        LBFGSB  & n/a & 1.947548404 & n/a & \cellhi 0\\
        Mosek Primal  & \multicolumn{4}{c}{Ran out of memory} \\
        CPLEX Primal  & \multicolumn{4}{c}{Ran out of memory} \\
        CPLEX Dual  & \multicolumn{4}{c}{Ran out of memory} \\ \midrule
        
        & \multicolumn{4}{c}{$n=5001$} \\ \midrule
        Project and Forget  & 3.946556152e-01 & 3.946556154e-01 & 2e-08 & \cellhi 0 \\
        Mosek Dual& 3.880175376e-01 & 3.880175255e-01 & \cellhi 1.2e-08 & 7.4e-11 \\
        LBFGSB  & n/a & 3.947709104e-01 & n/a & \cellhi 0 \\
        Mosek Primal  & \multicolumn{4}{c}{Ran out of memory} \\
        CPLEX Primal  & \multicolumn{4}{c}{Ran out of memory} \\
        CPLEX Dual  & \multicolumn{4}{c}{Ran out of memory} \\ \midrule 
        & \multicolumn{4}{c}{$n=10001$} \\ \midrule
        Project and Forget & $0.197687348$ & $0.197687348$ & \cellhi 2.0e-8 & \cellhi $0$ \\
        Mosek Dual  & \multicolumn{4}{c}{Ran out of memory} \\
        LBFGSB  & \multicolumn{4}{c}{Ran out of memory} \\
        Mosek Primal  & \multicolumn{4}{c}{Ran out of memory} \\
        CPLEX Primal  & \multicolumn{4}{c}{Ran out of memory} \\
        CPLEX Dual  & \multicolumn{4}{c}{Ran out of memory} \\ \midrule 
        & \multicolumn{4}{c}{$n=20001$} \\ \midrule
        Project and Forget & $0.0989355070$ & $0.0989355070$ & \cellhi 2.0e-8 & \cellhi$0$ \\
        Mosek Dual  & \multicolumn{4}{c}{Ran out of memory} \\
        LBFGSB  & \multicolumn{4}{c}{Ran out of memory} \\
        Mosek Primal  & \multicolumn{4}{c}{Ran out of memory} \\
        CPLEX Primal  & \multicolumn{4}{c}{Ran out of memory} \\
        CPLEX Dual  & \multicolumn{4}{c}{Ran out of memory} \\ \midrule 
        \bottomrule
     \end{tabular}
     \caption{Table showing the convergence details for the various solvers. }
     \label{table:conv}
 \end{table}

Again, we look at how ROD and FD evolve. Figure \ref{fig:drot-internal} shows that \textsc{Project and Forget} does fewer projections overall than cyclic Bregman does in one iteration. This is the same as the result that we saw for metric nearness. Here, FD is measured as the true feasibility error.

\begin{figure}
    \centering
    \subfigure[Plot for ROD versus number of projections for \textsc{Project and Forget} and Bregman's cyclic method]{\includegraphics[width=0.49\linewidth]{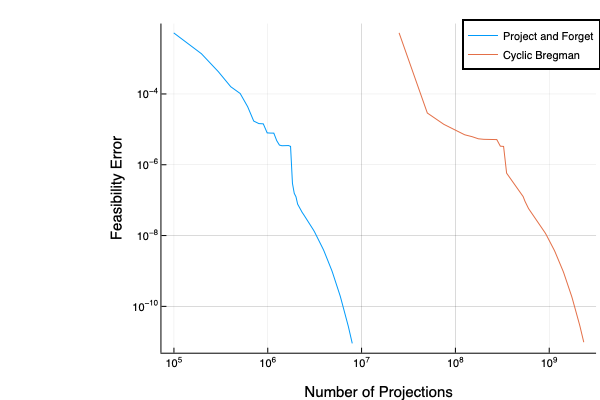}}
    \subfigure[Plot for ROD versus number of projections for \textsc{Project and Forget} and Bregman's cyclic method]{\includegraphics[width=0.49\linewidth]{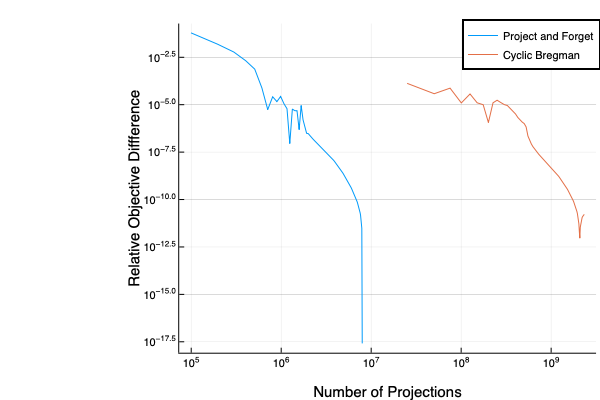}}
    \caption{Plot for ROD and FD versus number of projections for \textsc{Project and Forget} and Bregman's cyclic method}
    \label{fig:drot-internal}
\end{figure}

These results hold for synthetic data. We also test our method on real world data. For these tests, we consider the application of optimal transport known as color transfer (as done in \cite{blondel18a}). Here, we only compare against the best solver from the previous experiment. The results can be seen in Table \ref{table:color}. As we can see from the table, as we scale the size of the problem, \textsc{Project and Forget} takes less time than Mosek and we can scale to larger problems before running out of memory. We also have better convergence statistics than Mosek.

\begin{table}[]
    \centering
    \begin{tabular}{c|ccccccc}
    \toprule 
    Method & n & Time & Objective Gap & \multicolumn{2}{c}{Feasibility Error} \\ \midrule 
    & & & & Primal & Dual \\ \midrule 
    Project and Forget  & \multirow{2}{*}{512} &  611 & \cellhi 8e-11 & \cellhi 0 & \cellhi 9.8e-11 \\
    Mosek Dual & & \cellhi 6 & 2e-10 & 7e-8 & 7e-9\\ \midrule
    Project and Forget  & \multirow{2}{*}{1024} &  724 &  \cellhi 1.3e-10 & \cellhi 0 & \cellhi 5.2e-11 \\
    Mosek Dual & & \cellhi 31 & 9e-10 & 3e-8 & 2.9e-9\\ \midrule
    Project and Forget  & \multirow{2}{*}{4096} & \cellhi 825 &  \cellhi 2.8e-9 & \cellhi 0 & \cellhi 6.9e-9 \\
    Mosek Dual & & 995 & 2e-7 & 5.5e-7 & 2e-8\\ \midrule
    Project and Forget  & \multirow{2}{*}{8192} &\cellhi 2197 & \cellhi 3e-9 & \cellhi 0 & \cellhi 9.7e-8 \\ 
    Mosek Dual & & \multicolumn{4}{l}{Out of Memory} \\ \bottomrule
    \end{tabular}
    \caption{Table showing the time taken in seconds, the duality gap (difference between primal and dual objectives), as well as the primal and dual feasibility error for \textsc{Project and Forget} and Mosek for solving the dual formulation for color transfer. Here we set $\gamma = 1e2$}
    \label{table:color}
\end{table}

\subsection{Information Theoretic Metric Learning}

The next problem we consider is metric learning. There are many different versions of this problem (see \citet{metriclearning1, metriclearning2} for surveys on the topic).  We focus on a specific instantiation, information theoretic metric learning (ITML) from \cite{itml}. This formulation allows us to demonstrate the wide range of objective functions that our method can handle.  

Given two sets $S,D$ that represent the sets of similar and dissimilar points, we learn the co-variance matrix $M$ of a Gaussian distribution $p(x;M)$. The metric learned is then the Mahalanobis metric associated with $M$. Concretely, the problem is as follows:
 \begin{equation} 
 \begin{array}{ll@{}ll}
 \label{problem:itml}
 \text{minimize}  & KL\big(p(x;M) \| p(x; I)\big)  \\
 \text{subject to}& d_A(x_i,x_j) \le u  \ \ \ \ \  \ \ (i,j) \in S\\
                         & d_A(x_i,x_j) \ge l.  \ \ \ \ \  \ \ (i,j) \in D.
 \end{array}
 \end{equation}
 
\cite{itml} suggest solving this problem by sampling a small subset of the constraints and then using Bregman's cyclic method. However, since our method solves the problem with the large number of constraints, when we compare against the method from \cite{itml}, we can no longer compare in terms of convergence statistics or time taken to solve the problem. Hence, as the method from \cite{itml} is also based on Bregman projections, we limit both algorithms to the same number of projections. Then, we compare the algorithms on the quality of the solution; that is, the testing accuracy based on the metrics learned. 

For each data set, we uniformly at random choose $80\%$ of the data points to be the training set and the remaining to be the test set. We then let the similar pairs $S$ be those pairs that have the same label and the dissimilar pairs $D$ be all of the other pairs. For the algorithm from \cite{itml}, as suggested in the paper, we randomly sampled $20c^2$ constraints, where $c$ is the number of different classes and run the algorithm so that it performes approximately $10^6$ projections. We implement the algorithm in Julia. \textsc{Project and Forget} returns $50c^2$ uniformly sampled similarity constraints and $50c^2$ uniformly random dissimilarity constraints and runs until we perform approximately $10^6$ projections.

The hyper-parameters were set as follows: $\gamma = 1, u = 1, l=10$. The pseudo-code for our algorithm can be seen in Algorithm \ref{alg:itmlpf}. The classification is done using the $k$ nearest neighbor classifier, with $k=5$.

\begin{algorithm}[!htb]
\caption{Pseudo-code for the Project and Forget algorithm for ITML.}
\label{alg:itmlpf}
\begin{algorithmic}[1]
\Function{\textsc{PFITML}}{$X,C,\gamma,u,l,S,D$}
 \State $\lambda^0 = 0$, $\Xi_{ij} = u$ for $(i,j) \in S$ and $\Xi_{ij} = l$ for $(i,j) \in D$. Initialize $C= I$. 
    \While{Not done 1e7 projections}
	\State Randomly sample $50c^2$ $(i,j)$ from $S$ and $50c^2$ from $D$
	\State Do projection for these constraints.
	\State Project onto all remembered constraints.
    \EndWhile
    \Return $C$
\EndFunction

\Function{\textsc{Projection}}{$X,i,j,S,D,u,l,\Xi,\lambda,C$}
	\State $p = dist_C(X_i,X_j)$
	\State $\delta = 1$ if $(i,j) \in S$ and $\delta = -1$ if $(i,j) \in D$
	\State $\alpha = min\left(\lambda_{ij}, \frac{\delta}{2}\left(\frac{1}{p} - \frac{\gamma}{\Xi_{ij}}\right)\right)$
	\State $\beta = \frac{\delta\alpha}{1-\delta \alpha p}$
	\State $\Xi_{ij} = \gamma \frac{\Xi{ij}}{\gamma + \delta \alpha \Xi_{ij}}$
	\State $\lambda_{ij} = \lambda_{ij} - \alpha$
	\State $C = C + \beta C (x_i-x_j)(x_i-x_j)^TC^T$
\EndFunction
\end{algorithmic}
\end{algorithm}

 \begin{table}[!htb]
 \setlength{\tabcolsep}{4.5pt} 
 \centering
 \begin{tabular}{ccccccc}
 \toprule
 Dataset & |S| & |D| & \multicolumn{2}{c}{\textsc{Project and Forget}} & \multicolumn{2}{c}{ITML} \\
         & & &Test Accuracy & \#Projections & Test Accuracy & \#Projections \\
 \midrule
 Banana & 4.5e6 & 4.5e6 & 89.6\% & 1.003e6 & \cellhi 89.8\% & 1e6 \\
 Ionsphere & 2.1e4 & 1.8e4 & \cellhi 88.3\% & 1.007e6 & 83.7\% & 1e6 \\
 Coil2000 & 2.7e7 & 3.5e6 & \cellhi 93.3\% & 1.02e6 & 93.2\% & 1e6 \\
 Letter & 4.9e6 & 1.2e8 & 90.0\% & 1.08e6 & \cellhi 92.6\% & 1.0004e6\\
 Penbased & 3.9e6 & 3.5e7 & \cellhi 98.5\% & 1.08e6 & 98.2\% & 1e6\\
 Spambase & 3.5e6 & 3.2e6 & \cellhi 92.5\% & 1.01e6 & 90.7\% & 1e6  \\
 Texture & 8.8e5 & 8.8e6 & \cellhi 99.8\% & 1e6 & 98.7\% & 1.002e6 \\
 
 \bottomrule
 \end{tabular}
  \caption{Table comparing the testing accuracy of \textsc{Project and Forget} and \textsc{ITML}. Numbers are averaged over 5 trials.}
 \label{table:itml}
 \end{table}
 
As shown in Table \ref{table:itml}, our algorithm results in, for the most part, better test accuracy. The test accuracy is significantly better for the Ionosphere data set. This is the smallest of the data sets, from which we conclude that if \textsc{Project and Forget} can solve the entire problem, instead of the approximation from \cite{itml}, we get a performance boost. As the size of the problem gets larger, since we have limited the number of projections, our algorithm has not converged as much. On the other hand, since the size of the problem for \cite{itml} does not depend on the size of $S$ and $D$, its convergence does not change. 

We are also interested in how the test accuracy evolves as we do more projections. This can be seen in Figure \ref{fig:itml-internal}. As we can see since \textsc{ITML} solves a smaller problem; it very quickly stabilizes. However, since \textsc{Project and Forget} solves the complete problem using the stochastic variant of the method, as we do more projections, we see more data points and the accuracy evolves. In some cases, such as for Ionosphere, this results in increased accuracy. However, in other cases, such as Letter, this results in decreased accuracy. 

\begin{figure}
    \centering
    \subfigure[Inosphere]{\includegraphics[width = 0.49\linewidth]{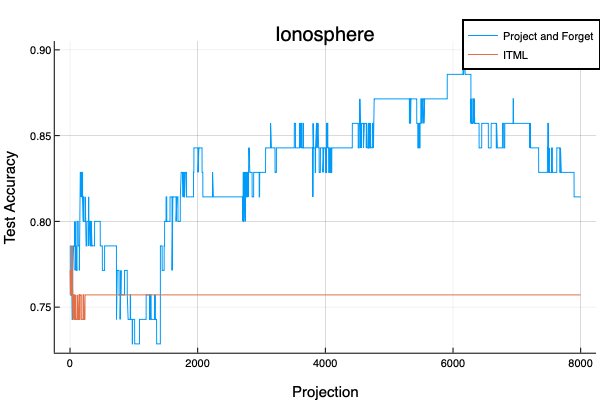}}
    \subfigure[Letter]{\includegraphics[width = 0.49\linewidth]{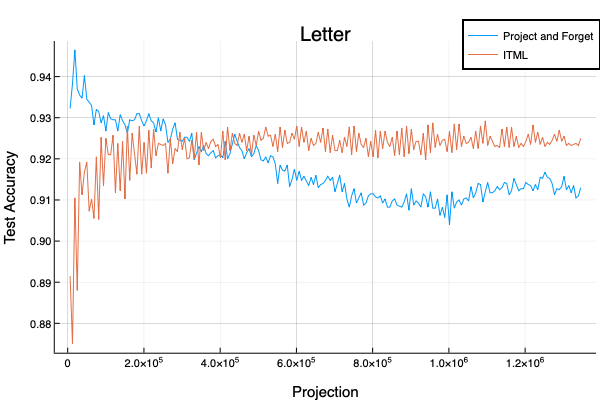}}
    \caption{Plots showing the test accuracy after each projection for \textsc{Project and Forget} (to solve the whole problem) and cylic Bregman to solve the restricted problem (\textsc{ITML}) for the Ionosphere and Letter dataset}
    \label{fig:itml-internal}
\end{figure}

\subsection{Support Vector Machines}

Finally, to show the usefulness of the truly stochastic variant, we use our algorithm to train $L_2$ SVMs. Given training data $x_1, \hdots, x_n$, a feature map $\phi$, and labels $y_1, \hdots, y_n \in \{-1,1\}$, we want to learn a vector $w$ and slack variables $\xi_i$ such that 
 \begin{equation*}
 \begin{array}{ll@{}ll}
 \text{minimize}  & \frac{1}{2} w^T w  + \frac{C}{2} \sum_i \xi_i^2 & \\
 \text{subject to}& y_i(w^T\phi(x_i)) \ge 1 -\xi_i &\ \ \ \ \forall \, i =1, \hdots n  \\
                  &                                                \xi_i \ge 0, &i=1 ,..., n.
 \end{array}
 \end{equation*}

Here we use an $L_2$ penalty for the slack variables instead of the usual $L_1$ penalty. \citet{Tang2013DeepLU} provides insights into when using an $L_2$ penalty would provide better results. 

\citet{Joachims2009CuttingplaneTO} is a cutting plane based method to train $L_2$ SVMs. However, in contrast with \citet{Joachims2009CuttingplaneTO}, we do not need to use an oracle that finds the maximum violated constraint. Instead, we used the truly stochastic variant of our algorithm. Hence, instead of searching through our data points to find large violations, we pick a subset of data points uniformly at random and project onto the constraints defined by these data points. 

\textbf{Implementation Details.}
We used the LIBLINEAR implementation that is available online. We interfaced with LIBLINEAR using an interface written in Julia. The implementation for our algorithm follows the pseudo-code presented in algorithm \ref{alg:imp4}.

 \begin{algorithm}[!htbp]
 \caption{Pseudo-code for the implementation training an SVM.}
 \label{alg:imp4}
 \begin{algorithmic}[1]
  \State $z^0,w^0 = 0$. $n$ is the number of training samples
     \For{$iter = 1, \hdots, MaxIters$}
     	\For{$i = 1, \hdots, n$}
 		    \State $x_j$ = random data point
 		    \State Project $w$ onto the constraint defined by $x_j$. 
 	    \EndFor
     \EndFor
     \Return $w$
 \end{algorithmic}
 \end{algorithm}

\textbf{Data.} We generated the data matrix $X$ as follows. Each entry $X_{ij}$ was sampled from $\mathcal{N}(0,K^2)$. We then sampled the coordinates of $H$ from $\mathcal{N}(0,1)$ and used $H$ to label all the points in $X$. We then added noise so that $H$ did not correctly separate the data. This noise was also sampled from $\mathcal{N}(0,1)$. We generate data in this manner because a standard pre-processing step is to normalize features (i.e., transform the data so that each feature has mean zero and variance one). 

We used $K= 1,2$ and $5$ to generate our data set. The different values for $K$ gave us the different values of the percentage of data points misclassified by the original classifier $s$ that we report in Table \ref{table:SVM}.

\textbf{Convergence criterion.} We considered a few different convergence criteria, but settled upon running the algorithm for 100 projections since whenever we train an SVM, we have a validation set. This number of projections becomes a hyperparamter in the model and, we found experimentally that increasing the number of projections (unless increasing it by several orders of magnitude) beyond 100 did not increase the validation accuracy.

\subsubsection{Results}

Table \ref{table:SVM} gives us the results for the experiment. In this experiment, we are not checking for convergence, but demonstrating a possible use case of the truly stochastic variant of our algorithm. Because of the way in which we have generated data, it is easy to see that we have good test accuracy only if the current separating hyper-plane is roughly similar to the optimal one. We also see that this method has roughly similar performance to the \textsc{LibLinear} dual solver. However, the primal solver from \textsc{Liblinear} has the best testing accuracy. Thus, we see that the truly stochastic version of the algorithm can potentially be used to get reasonable results quickly. 
 
\begin{table*}[h!]

 \centering
 \begin{tabular}{cccccccccc}
 \toprule
 \multicolumn{4}{c}{Parameters} &  \multicolumn{3}{c}{Time}  & \multicolumn{3}{c}{Test Acc.} \\ \midrule
  $n$ & $d$ & $s$ & $C$ & Ours &  \multicolumn{2}{c}{Liblinear} & Ours & \multicolumn{2}{c}{Liblinear}   \\
 & &  & & & Dual & Primal & & Dual & Primal  \\ \midrule
 1,000,000& 100 & $6.3\%$ & $10^3$ & \cellhi 1.33 s & 547 s  & 9.87 s & 94.7\% & 94.7\% & \cellhi 96.8\%  \\
 1,000,000& 100 & $12.6\%$ & $10^3$ & \cellhi 1.38 s & 1032 s &  10.1 s &91.4\% & 90.2\% & \cellhi  93.6\%  \\
 1,000,000& 100 & $29.5\%$ & $10^3$ & \cellhi 1.58 s & 1532 s & 8.32 s & 78.4\% & 77.1\%   & \cellhi 85.2\% \\
 \bottomrule
 \end{tabular}
  \caption{Table comparing the testing accuracy and running times for the truly stochastic variant of or algorithm against \textsc{LibLinear} for binary classification using an $L_2$ SVM.} 
 \label{table:SVM}
\end{table*}

%
%
%

\section{Conclusion and Future work} 

In conclusion, in this paper, we present a new algorithm \textsc{Project and Forget} that can be used to solve highly constrained convex optimization problems. We show that under some general assumptions, our method has a linear rate of convergence. Additionally, we demonstrate that our method returns both the dual and primal solutions.

The main type of problems we are interested in solving are metric constrained problems. We show that standard solvers cannot solve metric constrained problems at even moderate scales. We then demonstrate that \textsc{Project and Forget} can solve the problem at large scales. We do this by solving larger instances of the metric nearness and correlation clustering problems. 

We also demonstrate the generality of the method. First, we show that a variety of objective functions can be used with our methods, as demonstrated by the information theoretic metric learning experiment. We also demonstrate the role of sparsity and how it helps our method using the optimal transport experiment. 

For future work, we hope that our method inspires a new way of learning metrics on data sets. Earlier researchers were either restricted to learning embeddings and then extracting a metric from those embeddings. However, using our method, we can now learn a much wider variety of metrics. We also hope that once we can learn optimal metrics on data, these optimal metrics then suggest the types of spaces the data should be embedded into, rather than users picking the space in which to embed in an ad hoc fasion. We also hope to adapt the general convex constraint version to solve highly constrained semi-definite programs. 
\label{sec:conclusion}
\section{Proofs}

\subsection{Proof of part \ref{part:1prime} of Theorem \ref{thm:linear} for oracles that satisfy property \ref{prop:sep}} 
\label{sec:proof11}

We remind the reader of the notation established in Section~\ref{sec:prelims}. The vector of variables over which we optimize is $x$, $f$ is the objective function, 
$H_i = \{ y : \langle y, a_i \rangle = b_i \}$ are the hyper-planes that lie on the boundaries of the half-space constraints, $L(x,z)$ is the Lagrangian, $z$ is the dual variable, $A$ is the matrix with rows given by $a_i$, and $b$ is the vector with rows $b_i$. 

Next, we clarify the indexing of the variables. Algorithm \ref{alg:bregman} has three steps per iteration and during the \textsc{Project} step there are multiple projections. When we want to refer to a variable after the $\nu$th iteration, it will have a superscript with a $\nu$. When we refer to a variable after the $n$th, $i$th, $k$th projection, we use the superscript $n$,$i$,$k$. Finally, before the $n$th projection, $i(n)$ will represent the index of the hyper-plane onto which we project.

Finally, let $R$ be the maximum number of constraints that our oracle $\mathcal{Q}$ returns. This is clearly upper bounded by the total number of constraints, which we have assumed is finite.
We are now ready to prove the first part of Theorem~\ref{thm:linear}. 

{\reftheorem{thm:linear} (Part 1) If $f \in \mathcal{B}(S)$, $H_i$ are strongly zone consistent with respect to $f$,  $\exists\, x^0 \in S$, such that $\nabla f(x^0) = 0$, and the oracle $\mathcal{Q}$ satisfies property \ref{prop:sep}, then any sequence $x^n$ produced by Algorithm \ref{alg:bregman} converges to optimal solution of problem \ref{problem:linear}. }
\begin{proof} The proof of this theorem is an adaptation of the proof of convergence for the traditional Bregman method that is presented in \citet{book} whose proof entails the following four steps. The main difference between Censor and Zenios' proof and ours is that of the last two steps. We present the entire proof, however, for completeness. To that end, we show the following. 
\begin{itemize}[nosep]
	\item Step 1. The KKT condition, $\nabla f(x) = \nabla f(x^0) -A^T z$, is always maintained.
	\item Step 2. The sequence $x^n$ is bounded and it has at least one accumulation point.
	\item Step 3. Any accumulation point of $x^n$ is feasible (i.e., is in $C$).
	\item Step 4. Any accumulation point is the optimal solution. 
\end{itemize}

\noindent\textbf{Step 1.} The KKT condition, $\nabla f(x) = \nabla f(x^0) -A^T z$, is always maintained.

We show by induction that for all $n$, $\nabla f(x^n) = - A^T z^n$. In the base case, $z = 0$, thus, $\nabla f(x^0) = 0 = -A^T z^0$. Assume the result holds for iteration $n$, then

\begin{align*}
	\nabla f(x^{n+1}) &= \nabla f(x^n) + c^n a_{i(n)} \\
	                  &= -A^Tz^n + A^Tc^ne_{i(n)} \\
	                  &= -A^T(z^n - c^ne_{i(n)}) \\
	                  &= -A^Tz^{n+1}. 
\end{align*}

We know that $c^n \le z^n_{i(n)}$; therefore, we maintain $z^{n+1} \ge 0$ as well. 

\noindent\textbf{Step 2.} The sequence $x^n$ is bounded and has an accumulation point.

To show that $x^n$ is a bounded, we first show that $\big(L(x^n,z^n)\big)_n$ is a monotonically increasing sequence bounded from above. This observation results from the following string of equalities.

{\small
\begin{align*} 
L(x^{n+1},z^{n+1}) - L(x^n, z^n) &=f(x^{n+1}) - f(x^n) + \langle z^{n+1}, Ax^{n+1} - b\rangle  - \langle z^n,Ax^n-b\rangle \\
&= f(x^{n+1}) - f(x^n)  + \langle A^Tz^{n+1}, x^{n+1} \rangle  -  \langle A^Tz^{n}, x^{n} \rangle - \langle z^{n+1} - z^n, b \rangle \\
&= f(x^{n+1}) - f(x^n) - \langle \nabla f(x^{n+1}), x^{n+1} \rangle  \langle \nabla f (x^n), x^{n} \rangle + \langle c^n e_{i(n)}, b \rangle\\
&= f(x^{n+1}) - f(x^n) - \langle \nabla f(x^{n})+ c^na_{i(n)}, x^{n+1} \rangle  +  \langle \nabla f (x^n), x^{n} \rangle + c^n b_{i(n)} \\
&= f(x^{n+1}) - f(x^n) - \langle \nabla f(x^{n}) , x^{n+1} - x^n \rangle - \langle  c^na_{i(n)}, x^{n+1} \rangle + c^n b_{i(n)} \\
&= \underbrace{D_f(x^{n+1},x^n)}_{(1)}  + \underbrace{c^n(b_{i(n)} - \langle a_{i(n)},x^{n+1}\rangle)}_{(2)}. 
\end{align*}}

Next, we show that both terms (1) and (2) are non-negative. We know that $D_f$ is always non-negative so  we only need to consider term (2). There are two cases: (i) if $c^n = \theta^n$, then $x^{n+1} \in H_{i(n)}$ and $b_{i(n)} - \langle a_{i(n)},x^{n+1}\rangle = 0$. On the other hand, (ii) if $c^n = z^n_{i(n)}$, then $b_{i(n)} - \langle a_{i(n)},x^{n+1}\rangle \ge 0 $ and $c^n \ge 0$. We can conclude that the difference between successive terms of $L(x^n,z^n)$ is always non-negative and, hence, it is an increasing sequence. 

To bound the sequence, let $y$ be a feasible point (i.e., $Ay \le b$). (Note that this is the only place we use the assumption that the feasible set is not empty.) Then 
\begin{align*} 
D_f(y,x^n) &= f(y) - f(x^n) - \langle \nabla f(x^n), y-x^n \rangle \\
&= f(y) - f(x^n) + \langle z^n, Ay - Ax^n \rangle \\
&\le f(y) - f(x^n) + \langle z^n, b - Ax^n \rangle. 
\end{align*}
Rearranging terms in the inequality, we obtain a bound on the sequence $L(x^nz^n)$ from above:
\begin{align*}
	L(x^n, z^n) &= f(x^n) + \langle z^n, Ax^n -b \rangle \\
	&\le f(y) - D_f(y,x^n) \\
	&\le f(y).
\end{align*}

Since the sequence $(L(x^n,z^n))_{n \in \mathbb{N}}$ is increasing and bounded, it is a convergent sequence and the difference between successive terms of the sequence goes to 0. Therefore,  
\[ 
	\lim_{n \to \infty} D_f(x^{n+1}, x^n) = 0. 
\] 
From the previous inequality we also have that 
\[ 	
	D_f(y,x^n) \le f(y) - L(x^n,z^n) \le f(y) - L(x^0,z^0) =: \alpha. 
\] 
Using part (ii) of the definition of a Bregman function, we see that $L_2^f(y,\alpha)$ is bounded and since $(x^n)_{n \in \mathbb{N}} \in L_2^f(y,\alpha)$, $x^n$ is an infinite bounded sequence. Thus, has an accumulation point.  

\textbf{Step 3.}  Any accumulation point $x^*$ of $x^n$ is feasible (i.e., is in $C$).

This is the only step in which we use the fact that our oracle satisfies property \ref{prop:sep}. Let $x^*$ be some accumulation point for $x^n$ and assume for the sake of contradiction that $Ax^* \not \le b$. Let $\tilde{A}, \tilde{b}$ be the maximal set of constraints that $x^*$ does satisfy; i.e., 
\[ 
	\tilde{A}x^* \le \tilde{b}.
\]
Let $(x^{n_k})$ be a sub-sequence such that $x^{n_k} \to x^*$ and $H$ be a constraint that $x^*$ violates. Define $\epsilon$ as 
\begin{equation} \label{eq:epsilon} 
	\epsilon := \phi(d(x^*,H))  > 0. 
\end{equation}

Because $x^n$ is bounded, $x^{n_k}$ is a convergent sub-sequence $x^{n_k} \to x^*$, and $D_f(x^{n+1}, x^n) \to 0$, by Equation 6.48 from \cite{book} we see that for any $t$, 
\[ 
	x^{n_k + t} \to x^*. 
\] 
In particular, the proposition holds for all $t \le 2|\tilde{A}| + 2 =: T$.

Let us consider an augmented sub-sequence $x^{n_k},x^{n_k+1},\hdots, x^{n_k+T}$, i.e., add in extra terms. Note that if $n_{k+1}-n_k \to \infty$, then this augmented sequence is not the entire sequence. We want to show that infinitely many of the terms in our augmented sequence satisfy a constraint \emph{not} in $\tilde{A}$. Should this hold, then because we have only finitely many constraints, there exists at least one single constraint $\tilde{a}$ that is \emph{not} in $\tilde{A}$, such that infinitely many terms of the augmented sequence all satisfy the single constraint $\tilde{a}$. Finally, because our augmented sequence converges to $x^*$ and we are only looking at closed constraints, we must have that $x^*$ also satisfies the constraint $\tilde{a}$. Thus, we would arrive at a contradiction of the maximality of 
$\tilde{A}$ and $x^*$ would have to be in the feasible region. 

To see that infinitely many of the terms in our augmented sequence satisfy a constraint \emph{not} in $\tilde{A}$, let $\nu_k$ be the iteration in which the $n_k$th projection takes place. Note that we can assume without loss of generality that in any iteration, we project onto any constraint at most once. If this were not the case and we projected onto constraints more often, we would simply change the value of  $T$ to reflect this larger number of projections. Therefore, we have two possibilities for which iteration the $n_k + |\tilde{A}|+1$st projection takes place and we consider each case below.

\textbf{Case 1:} The $n_k + |\tilde{A}|+1$st projection, infinitely often, takes places in $\nu_k$th iteration. Since we project onto each constraint at most once, one of the projections between the $n_k$ and $n_k + |\tilde{A}| + 1$st projection must be onto a hyper-plane defined by a constraint \emph{not} represented in $\tilde{A}$ and amongst the terms $x^{n_k},x^{n_k+1},\hdots, x^{n_k+|\tilde{A}|+1}$, we must have a term that satisfies a constraint \emph{not} in $\tilde{A}$ infinitely often. 

\textbf{Case 2:} The $n_k + |\tilde{A}|+1$st projection, infinitely often, takes place in $\nu_k+1$st iteration or later. 

If this projection happens in $\nu_k+1$st iteration, consider the iteration in which we do the $n_k + T$th projection. If this projection also takes place in the $\nu_k+1$st iteration, then we have done at least $|\tilde{A}| + 1$ projections in the $\nu_k+1$st iteration. Hence, amongst $x^{n_k+|\tilde{A}|+1}, \hdots, x^{n_k+T}$, we must have a term that satisfies a constraint \emph{not} in $\tilde{A}$. 

If the $n_k + |\tilde{A}|+1$st or the $n_k +T$th projection happens in the $\nu_k+2$nd iteration or later, then between the $n_k$th and the $n_k +T$th projection, we must have projected onto all constraints returned by oracle in the $\nu_k+1$st iteration. Therefore, we must have projected onto some hyper-plane defined by $\hat{a}^{n_k}$ (for some constraint $\hat{C}^{n_k}$) such that
\[ 
	\hat{d}^{n_k} := d(x^{\nu_k+1}, \hat{C}^{n_k}) \ge \phi(d(x^{\nu_k+1},C)). 
\]  
Then there exists a sufficiently small $\delta > 0$, depending on $\tilde{A}, \tilde{b}, x^*$, such that if $\|y - x^*\| \le \delta$, then 
\[ 
	\tilde{A} y \le \tilde{b} + \frac{\epsilon}{2} \mathbbm{1}, 
\] 
where $\mathbbm{1}$ is vector of all ones. 

Since our augmented sequence converges to $x^*$, we know that there exists a $K$, such that for all $k \ge K$ and $t \le T$, $\|x^{n_k+t} - x^*\| < \delta$. That is, for all $k \ge K$ and $t \le T$,  
\begin{equation} \label{eq:assump} 
	\tilde{A} x^{n_k+t} \le \tilde{b} + \frac{\epsilon}{2} \mathbbm{1}.
\end{equation} 
Note $x^{\nu_k+1}$ is within our augmented sequence so if $\hat{a}$ is infinitely often in $\tilde{A}$, by \eqref{eq:assump}, we have that infinitely often 
\[ 
	\frac{\epsilon}{2} \ge \hat{d}^{n_k}. 
\] 
Finally, because the augmented sequence converges to $x^*$, 
\begin{align*}
	\epsilon &= \phi(d(x^*,H)) \\
	&\le \phi(d(x^*,C)) \\
	&= \lim_{k \to \infty}  \phi(d(x^{\nu_k+1},C)) \\
	&\le \lim_{k \to \infty} \hat{d}^{n_k} \\
	&\le \frac{\epsilon}{2}.
\end{align*}
 The first inequality follows from the fact that $\phi$ is non-decreasing. Thus , we have a contradiction. Therefore, $\hat{a}^{n_k}$ is \emph{not} in $\tilde{A}$ infinitely often and amongst $x^{n_k},x^{n_k+1},\hdots, x^{n_k+T}$, we must have a term that satisfies a constraint \emph{not} in $\tilde{A}$ infinitely often. Thus, there is a constraint $\tilde{a}$ not in $\tilde{A}$ that is satisfied by infinitely terms of our augmented sequence and we have a contradiction.

\noindent\textbf{Step 4.} Optimality of accumulation point.

Because we have established the feasibility of all accumulation points, we show next that any accumulation point $x^{n_k} \to x^*$ is optimal. 

First, we show that there exists an $N$, such that for any $k \ge N$, and  for any $a_i$ such that 
\[ 
	\langle a_i,x^* \rangle < b_i ,
\] 
we have $z_i^{n_k} = 0$. To do so, we assume for the sake of contradiction that for some $a_i$, our sequence $z^{n_k}_i$ is infinitely often not 0.  The algorithm then projects onto this constraint infinitely often. Therefore, the point $x^{n_k}$ lies on the hyper-plane defined by $a_i,b_i$ infinitely often. Thus, the limit point $x^*$ must lie on this hyper-plane as well and we have a contradiction. 

Now we know that for any constraint $a_i$, we either have that $\langle a_i, x^*\rangle = b_i$ or we have that $z_i^{n_k} = 0$ for the tail of the sequence. Thus, for sufficiently large $k$, 
 \begin{align*} 
 	\langle z^{n_k} , Ax^{n_k} -b \rangle &= \langle A^T z^{n_k} , x^{n_k} - x^* \rangle \\
 	&= \langle - \nabla f(x^{n_k})  , x^{n_k} - x^* \rangle  \\
 	&= D_f(x^*,x^{n_k}) - f(x^*) + f(x^{n_k}).
 \end{align*}
 Next, by part (iii) of the definition of a Bregman function, 
 \[ 
 	\lim_{k \to \infty} D_f(x^*,x^{n_k}) = 0. 
 \] 

Finally, 
\[ 
	\lim_{k \to \infty} L(x^k,z^k) = \lim_{k \to \infty} f(x^{n_k}) +   
	\langle z^{n_k} , Ax^{n_k} -b \rangle  = f(x^*). 
\] 
We also know that $L(x^k,z^k) \le f(y)$ for any feasible $y$. Thus, $f(x^*) \le f(y)$. Hence $x^*$ is an optimal solution. Now since $f$ is strictly convex, this optimal point is unique. Therefore, we have that $(x^n)_{n \in \mathbb{N}}$ has only one accumulation point and $x^n \to x^*$.  
\end{proof}

An important consequence of this proof is the following proposition: 

{\refprop{prop:active} If $a_i$ is an inactive constraint, then there exists a $N$, such that for all $n \ge N$, we have that $z^n_i=0$. That is, after some finite time, we never project onto inactive constraints ever again. }

\subsection{Proof of part \ref{part:1prime} of Theorem \ref{thm:linear} for oracles that satisfy property \ref{prop:rand}} \label{sec:proof12}

In this subsection, we prove part \ref{part:1prime} of Theorem \ref{thm:linear} for oracles that satisfy property \ref{prop:rand}. We make note of the key ideas in this proof as they are useful in the proof of the truly stochastic variant. To be precise, we prove:

{\reftheorem{thm:linear} (Part 1) If $f \in \mathcal{B}(S)$, $H_i$ are strongly zone consistent with respect to $f$,  $\exists\, x^0 \in S$, such that $\nabla f(x^0) = 0$, and the oracle $\mathcal{Q}$ satisfies property \ref{prop:rand}, then with probability $1$, any sequence $x^n$ produced by Algorithm \ref{alg:bregman} converges to optimal solution of problem \ref{problem:linear}. }

\begin{proof} Assume that we have an oracle that satisfies property \ref{prop:rand}. A careful reading of the previous proof shows that if we switch out an oracle with property \ref{prop:sep} for an oracle with property \ref{prop:rand}, then we only need to adjust step 3 of our proof. The crucial part of that step was showing that for our augmented sequence, we had infinitely many terms that satisfied a constraint \emph{not} in $\tilde{A}$. We make the following adjustments to our analysis.

Let $\nu_k$ be the iteration in which the $n_k$th projection takes place. In the previous proof, we used the property of the oracle only when the $n_k + T$th projection took place in the $\nu_k+2$nd iteration or a later iteration. In this case, the augmented sequence encompasses all of the $\nu_k+1$st iteration infinitely often. 

Let us choose a constraint $\hat{a}$ that is not satisfied by $x^*$. Because the oracle satisfies property \ref{prop:rand}, for each iteration $\nu_k+1$, our oracle returns $\hat{a}$ with probability at least $\tau > 0$. By the Borel Cantelli Lemma, we know that during the selected iterations, the constraint $\hat{a}$, with probability one, is returned infinitely often by our oracle. Thus, our augmented sequence satisfies this constraint with probability $1$ and $x^*$ lies in the feasible region with probability $1$. 
\end{proof}

A direct consequence of this proof is the proof of the probabilistic version of Proposition \ref{prop:active}.

{\refprop{prop:active} With probability $1$, we project onto inactive constraints a finite number of times.}

\subsection{Proof of part \ref{part:3prime} of Theorem \ref{thm:linear}} \label{sec:proofpart2}

The discussions in \citet{IusemH, IusemB} almost directly apply to that for our algorithm. For completeness, we present it along with the necessary modifications. As with the traditional Bregman algorithm, we first present the case when $f(x)$ is quadratic. That is, 
\[ 
	f(x) = r + s^T \cdot x + \frac{1}{2}x^T Hx, 
\] 
where $H$ is a positive definite matrix. In this case, it is easy to see that 
\[ 
	D_f(x,y) = \|x-y\|_H^2 := (x-y)^T H (x-y).
\]

\subsubsection{Proof of part \ref{part:3prime} of Theorem \ref{thm:linear}---Quadratic Case}

In this section, we will prove the following variation of Theorem \ref{thm:linear}. 

{\reftheorem{thm:linear} If $f$ is a strictly convex quadratic function, $H_i$ are strongly zone consistent with respect to $f$,  $x^0 = H^{-1} s \in S$, and the oracle $\mathcal{Q}$ satisfies either property \ref{prop:sep} or \ref{prop:rand}, then there exists $\rho \in (0,1)$ such that 
\begin{equation} 
\tag{\ref{eq:conv}} \lim_{\nu \to \infty } \frac{\|x^* - x^{\nu+1} \|_H}{\|x^* - x^\nu \|_H} \le \rho,  
\end{equation} where $\|y\|_H ^2 = y^THy$. In the case when we have an oracle that satisfies property \ref{prop:rand}, the limit in \ref{eq:conv} holds with probability $1$. }

We establish some notation ahead of our lemmas. Let $I$ be the set of all active constraints. That is, if $x^*$ is the optimal solution then 
\[ 
	I = \{ i : \langle a_i , x^*\rangle = b_i \}.
\]
Let $S$ be the set of all $x$ that satisfy these constraints (namely $S = \{ x : \forall i \in I, \, \langle a_i, x\rangle = b_i \}$). Let $H_x$ be the hyper-plane, such that $H_x$ represents the constraint in $I$ that is furthest from $x$.  Define 
\[ 
	\mu = \inf_{x \not\in S} \frac{d(x, H_x)}{d(x,S)}. 
\] 
By \citet{IusemH}, we know that $\mu > 0$. Let $U$ be the set of all optimal dual variables $z$; i.e.,  $U = \{ z : \nabla f(x^*) = -A^Tz \}$ and let $I_\nu =  \{ i : z^{\nu+1}_i \neq 0\}$.

Next, we present a few preliminary lemmas. These lemmas exist in some form or another in \citet{IusemB, IusemH} and we present them suitably modified for our purpose. These lemmas require the following set of assumptions about an iteration $\nu$:
\begin{enumerate} 
	\item $\forall i \not\in I$, $z_i^\nu = 0$;
	\item for all $i \not\in I$, we do not project onto this constraint in the $\nu$th iteration; and,
	\item there exists $z \in U$, such that for all $i \not\in I_\nu$, $z_i = 0$. 
\end{enumerate}

\begin{lemma} \label{lem:converse} Let $x^*$ be the optimal solution for an instance of problem \ref{problem:linear}. For any sequence $x^n \to x^*$ such that $x^n,z^n$ maintain the KKT conditions, there exists an $M$, such that for all $\nu \ge M$, there exists a $z \in U$, such that for all $i \not\in I_\nu$, we have that $z_i = 0$. 
\end{lemma}
\begin{proof} Let $V_\nu = \{ z : \forall i \not\in  I_\nu, z_i = 0\}$. Then assume, for the sake of contradiction, that the result is false. That is, there is a sequence $\nu_k$ such that $V_{\nu_k} \cap U = \emptyset$. Then since there finitely many different $I_\nu$ (hence finitely many $V_\nu$), we have that one of these must occur infinitely often. Thus, by taking an appropriate sub-sequence, we assume, without loss of generality, that $I_{\nu_k}$ are all equal. Let $V=V_{\nu_k}$ and obtain $V \cap U = \emptyset$. 

Since $V$ is a closed subspace, $U$ is a closed set, and $V \cap U = \emptyset$, we must have that $d(V,U) > 0$. But $z^{\nu_k+1} \in V$ and so 
\[ 
	d(z^{\nu_k+C}, U) \ge d(V,U) > 0.
\] 
Since $x^\nu \to x^*$ and we maintain the KKT conditions, we have that for any $z \in U$, 
\[
	A^Tz^{\nu} = - \nabla f(x^\nu)  \to - \nabla f(x^*) = A^Tz.
\] 
Thus $d(z^\nu, U) \to 0$ which is a contradiction. 
\end{proof}

\begin{lemma} \label{lem:ineq}  For any sequence $x^n \to x^*$, if for a given $\nu$, we have that the sequence satisfies assumptions (1) and (2), then 
\[ 
	\| x^{\nu+1} -  x^*\|_Q^2 \le \|x^\nu - x^*\|_Q^2 -  \sum_{n= k}^K \|x^{n+1}-x^{n}\|_Q^2
\] 
where $k$ and $K$ are the indices of the first and last projection that take place in the $\nu$th iteration.
\end{lemma}
\begin{proof} This Lemma is simply a statement about Bregman projections and so its proof requires no modification.
\end{proof}

Before we proceed, we introduce additional notation. Let $A_{I_\nu}, b_{I_\nu}$ be the sub-matrix of $A,b$ with rows from $I_\nu$ and 
\[
	S_\nu = \{ x: A_{I_{\nu_k}}x = b_{I_{\nu_k}} \}.
\]

\begin{lemma} \label{lem:strongerdist} For any sequence such that $x^n \to x^*$, if for a given $\nu$, we have that it satisfies assumptions (1), (2), and (3), then we have that $\|x^{\nu+1} - x^*\|_Q = d(x^{\nu+1}, S_\nu)$. \end{lemma}

\begin{proof} Consider the constrained problem 
\begin{equation} \label{eq:q} 
    \min_{x \in S_\nu} \|x^{\nu+1} - x\|_Q^2. 
\end{equation} 
Then sufficient conditions for a pair $(x,z_{I_\nu})$ to be optimal for this problem are 
\[ 
    A_{I_\nu} x = b_{I_\nu} \text{ and } x = x^{\nu+1} - Q^{-1} A^T_{I_\nu} z_{I_\nu}. 
\] 
By Proposition \ref{prop:active}, we see that since $x^*$ is solution to problem \ref{problem:linear}, we have that $A_{I_\nu} x^* = b_{I_\nu}$. Then by assumptions and the manner in which we do projections, we have that there exists $z \in U$, such that for all $i \not \in I_\nu$, $z_i = 0$ and 
\[
    x^* = x^{\nu+1} - Q^{-1}A^T(z^{\nu+1} - z).
\]

Then since $z^{\nu+1}_i = 0$ for all $i \not \in I_\nu$, we have that  
\[ 
    x^* = x^{\nu+1} - Q^{-1}A^T_{I_\nu} (z^{\nu+1}_{I_\nu}  - z_{I_\nu} ). 
\]  
Thus, $x^*$ is the optimal solution to \ref{eq:q}. \end{proof}

Next for $x \not\in S_\nu$, let $H_x^\nu$ be the hyper-plane of that is furthest from $x$ and define 
\[ 
    \mu_\nu = \inf_{x \not \in S_\nu} \frac{d(x,H_x^\nu)}{d(x,S_\nu)}. 
\]

Now we are ready to prove the following theorem. 

\begin{thm} \label{thm:linearIter} Let $x^*$ is the optimal solution to problem \ref{problem:linear}. Then given $\nu$ that satisfies assumptions (1), (2), and (3), we have that \[ \|x^{\nu+1} - x^*\|_Q^2 \le \frac{L}{L+\mu^2} \|x^\nu - x^*\|_Q^2 \] where $L$ is the number of projections that happened in $\nu$th iteration.  \end{thm}

\begin{proof} By Lemma \ref{lem:strongerdist}, for any such $\nu$ we have that $x^{\nu+1} \not\in S_\nu$ (or we have converged already). Suppose constraint $j \in I_\nu$ defines the hyper-plane $H^{\nu}_{x^{\nu+1}}$. Then by Lemma \ref{lem:strongerdist} and definitions of $\mu_\nu, \mu$ we have the following inequality. 
\begin{align*} 
\|x^{\nu+1} - x^* \| &= d(x^{\nu+1}, S_\nu) \\
&\le \frac{1}{\mu_\nu}  d(x^{\nu+1},H_{x^{\nu+1}}^\nu)  \\
&\le \frac{1}{\mu}d(x^{\nu+1},H_{x^{\nu+1}}^\nu).
\end{align*}
Now since $I_\nu =  \{ i : z^{\nu+1}_i \neq 0\}$, we know that during the $\nu$th iteration we must have projected onto $H^{\nu}_{x^{\nu+1}}$. Note that this is the only place in the proof where we need the fact that we remember old constraints. Let us say that this happens during the $r$th projection of the $\nu$th iteration. 

Note by assumption, we satisfy the assumptions of Lemma \ref{lem:ineq}. Let $y^r,y^{\nu+1}$ be the projections of $x^r, x^{\nu+1}$ onto $H^\nu_{x^{\nu+1}}$. Then we see that

\begin{align*} d(x^{\nu+1},H_{x^{\nu+1}}^\nu)^2 &= \|y^{\nu+1} - x^{\nu+1} \|_Q^2\\
&\le \|y^r - x^{\nu+1} \|_Q^2 \\
&\le \left(\|y^r - x^{r+1} \|_Q + \sum_{i=r+1}^L \|x^i - x^{i+1} \|_Q \right)^2  \\
&= \left(\sum_{i=r+1}^L \|x^i - x^{i+1} \|_Q\right)^2 \\
&\le \left(\sum_{i=0}^L \|x^i - x^{i+1} \|_Q \right)^2 \\
&\le L \sum_{i=0}^L \|x^i - x^{i+1}\|_Q^2 \\
&\le L \left(\|x^\nu-x^*\|_Q^2 - \|x^{\nu+1} - x^*\|_Q^2\right). 
\end{align*}

Thus, we get that 
\begin{align*} 
    \mu^2\|x^{\nu+1}x^*\|^2 &\le d(x^{\nu+1},H_{x^{\nu+1}}^\nu)^2  \\
                            &\le L\left(\|x^\nu-x^*\|^2 - \|x^{\nu+1} - x^*\|_Q^2\right).
\end{align*}

Rearranging, we get that \[ \|x^{\nu+1}-x^*\|_Q^2 \le \frac{L}{L+\mu^2} \|x^\nu-x^*\|_Q^2. \]
\end{proof}

As a corollary to the above theorem, we have that algorithm 1 converges linearly. 

{\reftheorem{thm:linear} \ref{part:3prime}) If $f$ is a strictly convex quadratic function, $H_i$ are strongly zone consistent with respect to $f$,  $x^0 = H^{-1} s \in S$, and the oracle $\mathcal{Q}$ satisfies either property \ref{prop:sep} or \ref{prop:rand}, then there exists $\rho \in (0,1)$ such that \begin{equation}  \lim_{\nu \to \infty } \frac{\|x^* - x^{\nu+1} \|_H}{\|x^* - x^\nu \|_H} \le \rho \tag{\ref{eq:conv}} \end{equation} where $\|y\|_H ^2 = y^THy$. In the case when we have an oracle that satisfies property \ref{prop:rand}, the limit in \ref{eq:conv} holds with probability $1$. }

\begin{proof} Using Proposition \ref{prop:active}, Lemma \ref{lem:converse}, and that we have finitely many constraints, we see that if $\nu$ is large enough, the assumptions for Theorem \ref{thm:linearIter} are satisfied. Taking the limit gives us the needed result. 

In the case when we have an oracle that satisfies property \ref{prop:rand}, consider the product space of all possible sequences of hyper-planes returned by our oracle. In this product space, we see that with probability $1$, we generate a sequence of hyper-planes, such that algorithm \ref{alg:bregman} converges. For any such sequence of hyper-planes, we have that \ref{eq:conv} holds. Thus, the limit in \ref{eq:conv} holds with probability $1$ for random separation oracles. 
\end{proof} 

\subsubsection{Proof of part \ref{part:3prime} of Theorem \ref{thm:linear}---General}

The rate of convergence for the general Bregman method was established in \cite{IusemB}. To show this, let $\tilde{f}$ be the 2nd degree Taylor polynomial of $f$ centered at the optimal solution $x^*$. \[ \tilde{f}(x) = f(x^*) + \nabla f(x^*)^T \cdot x + \frac{1}{2} x^T \cdot  \nabla^2 f(x^*) \cdot  x. \] For notational convenience, let  $H$ be the Hessian of $f$ at $x^*$. Then we can see that if replace $f$ with $\tilde{f}$ in \ref{problem:linear} then the optimal solution does not change. Thus, if had access to $\tilde{f}$ and could use this function to do our projections, then from the quadratic case we have our result. 

Thus, to get the general result, if $x^\nu$ is our standard iterate and $\tilde{x}^\nu$ is the iterate produced by using $\tilde{f}$ instead of $f$, then \cite{IusemB} shows that $\|x^\nu - \tilde{x}^\nu\|$ is $o(\|x^\nu - x^*\|_H)$. Specifically, we can extract the following theorem from \cite{IusemB}.  

\begin{thm} \cite{IusemB} \label{thm:op} Let $x^*$ is the optimal solution for problem \ref{problem:linear} and $\tilde{x}^n$ is the sequence produced by using the same sequence of hyper-planes but with $\tilde{f}$ instead of $f$. Given a sequence $x^n$ produced by Bregman projections, such that $x^n \to x^*$, and for large enough $\nu$ we satisfy assumptions (1), (2), and (3), then $\|x^\nu - \tilde{x}^\nu\|$ is $o(\|x^\nu - x^*\|_H)$ \end{thm}

Using this we can get the general result as follows 

\begin{align*} \|x^{\nu+1} - x^*\|_H &\le \|x^{\nu+1} - \tilde{x}^{\nu+1}\|_H + \|\tilde{x}^{\nu+1} - x^*\|_H \\
&\le  \|x^{\nu+1} - \tilde{x}^{\nu+1}\|_H + \rho \|\tilde{x}^{\nu} - x^*\|_H  \\
&\le  \|x^{\nu+1} - \tilde{x}^{\nu+1}\|_H + \rho \|\tilde{x}^{\nu} - x^{\nu}\|_H \\
& \ \ \ \ + \rho\|x^{\nu} - x^*\|_H. \end{align*}

Then diving by $\|x^{\nu+1} - x^*\|_H$, and using Theorem \ref{thm:op} to take the limit, we get that there exists $\rho \in (0,1)$ such that 
\begin{equation} 
    \lim_{\nu \to \infty } \frac{\|x^* - x^{\nu+1} \|_H}{\|x^* - x^\nu \|_H} \le \rho. \tag{\ref{eq:conv}}  
\end{equation} 
As with the quadratic case, we see that is an oracle satisfies property \ref{prop:rand}, then \ref{eq:conv} holds with probability $1$. Thus, we have proved Theorem \ref{thm:linear} in its complete generality.

\subsection{Proof of Theorem \ref{thm:linear-stochastic}} \label{sec:proof-linear-stochastic}

In this section we prove Theorem \ref{thm:linear-stochastic} in essentially the same manner as we did for Theorem \ref{thm:linear} and so we outline only what changes are necessary. 

{\reftheorem{thm:linear-stochastic} If $f \in \mathcal{B}(S)$, the hyper-planes $H_i$ are strongly zone consistent with respect to $f$, and $\exists\, x^0 \in S$ such that $\nabla f(x^0) = 0$, then with probability $1$ any sequence $x^n$ produced by the algorithm converges to the optimal solution of problem \ref{problem:linear}. Furthermore, if $x^*$ is the optimal solution, $f$ is twice differentiable at $x^*$, and the Hessian $H := H f(x^*)$ is positive semi-definite, then there exists $\rho \in (0,1)$ such that with probability $1$, 
\begin{equation}
	\liminf_{\nu \to \infty } \frac{\|x^* - x^{\nu+1} \|_H}{\|x^* - x^\nu \|_H} \le \rho. \tag{\ref{eq:conv2}}
\end{equation} }

To prove Theorem \ref{thm:linear-stochastic}, we need to analyze only what goes wrong if the algorithm ``forgets'' all of the old constraints. First, consider the proof in the case that we converge to the optimal solution. Then, steps 1,2, and 3 are completely unaffected by forgetting old constraints. The only step that is affected is step 4. In a previous proof, we argued that if for some inactive constraint $a_i$, $z_i$  is non-zero infinitely often, then we projected onto this constraint infinitely often. In our present setting, we cannot conclude this directly as $z_i^\nu > 0$ does not imply that we remember $a_i$ on the $\nu$th iteration. However, due to property \ref{prop:rand}, we know that $\mathcal{Q}$ returns $a_i$ with probability at least $\tau$. Thus, again using the Borel Cantelli Lemmas, we see that we have $a_i$ infinitely often and this iteration converges to the optimal.

To prove the second part of the theorem, we recall from Theorem \ref{thm:op} that we only need to analyze the case when $f$ is a quadratic function. Indeed, the only place where we used the fact that we remembered old constraints was in the proof of Theorem \ref{thm:linearIter} in which we needed to remember old constraints to guarantee that during the $\nu$th iteration we project onto the constraint $a_i$ that is furthest from $x^\nu$ among those constraints for which $z_i^{\nu+1} > 0$. We cannot guarantee that this happens always but we can guarantee that it happens infinitely often.

Therefore, the conclusion of Theorem \ref{thm:linearIter} holds infinitely often instead of for the tail of the sequence and we replace the limit with a limit infimum to obtain the desired result. 

\subsection{General Convex Proof}

We will need the following facts from \cite{dykstra}.

\begin{enumerate}
\item \label{fact:1} A convex function if Legendre if and only if its convex conjugate is Legendre. In this case the gradient mapping 
\[ \nabla f : int(dom(f)) \to int(dom f^*) : x \mapsto \nabla f(x) \] is a topological isomorphism with inverse mapping $(\nabla f)^{-1} = \nabla f^*$. 
\item \label{fact:2} Suppose $f$ is Legendre on $E$ and a $S$ is a closed convex set in $E$ with $S \cap int(dom(f)) \neq \emptyset$. Suppose further $y \in int(dom(f))$. Then the Bregman projection $P_S^{f}y$ of $y$ onto $S$ with respect to $f$ is characterized by \[ P_S^fy \in S \cap int(dom(f)) \text{ and } \langle \nabla f(y) - \nabla f(P_S^fy), S - P_S^fy\rangle \le 0. \]

In addition, $D_f(P_S^fy,y) \le D_f(s,y) - D_f(s,P_S^fy)$ for all $s \in S \cap dom(f)$. Here when we write $\langle \nabla f(y) - \nabla f(P_S^fy), S - P_S^fy\rangle \le 0$, we mean that $\langle \nabla f(y) - \nabla f(P_S^fy), s - P_S^fy\rangle \le 0$, is true for all $s \in S$. 
\item \label{fact:3} Three point identity: Suppose $f$ is Legendre on $E$. If $x,y \in int(dom f)$ and $b \in dom f$, then \[ D_f(b,x) + D_f(x,y) - D_f(b,y) = \langle \nabla f(x) - \nabla f(y, x-b)\rangle. \]
\item \label{fact:4} Suppose $f$ is a very strictly convex function on $E$. Then for every compact convex subset $K$ of $int(dom f)$, there exists reals $0< \theta$ and $\Theta < +\infty$ such that for every $x, y \in K$ we have that \[ D_f(x,y) \ge \theta \|x-y\|^2 \text{ and } \|\nabla f(x) - \nabla f(y)\| \le \Theta \|x-y\|. \]
\end{enumerate}

On the $nth$ iteration we projection onto the constraint given by $i(n)$. Thus, we have that
\begin{equation}  x^n \in C_{i(n)} \cap dom(f). \label{eq:1} \end{equation} Then since $f$ is Legendre, using Facts \ref{fact:1} and \ref{fact:2} with $y = \nabla f^*(\nabla f(x^{n-1}) + q^{p(n)})$, $S = C_{i(n)}$ and $P_S^fy = x^n$, we have that \[ \langle \nabla f(x^{n-1}) + q^{p(n)} - \nabla f(x^n), C_{i(n)} - x^n \rangle \le 0. \] Then by definition of $q^{n}$ we have that \begin{equation} \label{eq:2} \langle q^{n} , x^n - C_{i(n)}  \rangle \ge 0. \end{equation}
Then we can also see that \begin{equation} \label{eq:tel} \nabla f(x^{n-1}) - \nabla f(x^n) = q^n - q^{p(n)}. \end{equation} Then summing over this telescoping sum and the fact that the $q$ are initialized to be 0 we get that \begin{equation} \label{eq:4} \nabla f(x^0) - \nabla f(x^n) = \sum_{k \in \mathcal{C}} q^{p(k,n)}. \end{equation} Here $\mathcal{C}$ stores the indices of all the constraints. To see why this is true, consider what happens if we sum the terms for just one constraint set. If we see the right hand side of Equation \ref{eq:tel}, then we see that only the $q$ from the last time we projected onto that set remains.

Let $c \in int(dom f)$. Then we need to prove the following crucial equality. 
\begin{align*} D_f(c,x^n) &= D_f(c,x^{n+1}) + D_f(x^{n+1}, x^n) - \langle \nabla f(x^{n+1}) - \nabla f(x^n), x^{n+1} - c \rangle &[\text{Fact \ref{fact:3}}] \\
&=  D_f(c,x^{n+1}) + D_f(x^{n+1}, x^n) + \langle q^{(n+1)} - q^{p(n+1)}, x^{n+1} - c \rangle &[\text{Eq \ref{eq:tel}}] \\
&=  D_f(c,x^{n+1}) + D_f(x^{n+1}, x^n) + \langle q^{n+1}, x^{n+1}-c\rangle  - \langle q^{p(n+1)}, x^{n+1} - c\rangle \\
&=  D_f(c,x^{n+1}) + D_f(x^{n+1}, x^n) + \langle q^{n+1}, x^{n+1}-c\rangle \\& \qquad \qquad \qquad \qquad - \langle q^{p(n+1)}, x^{n+1} + x^{p(n+1)} - x^{p(n+1)} - c\rangle \\
&=  D_f(c,x^{n+1}) + D_f(x^{n+1}, x^n) + \langle q^{n+1}, x^{n+1}-c\rangle \\& \qquad \qquad \qquad \qquad  - \langle q^{p(n+1)}, x^{n+1} - x^{p(n+1)} \rangle  + \langle q^{p(n+1)}, x^{p(n+1)} - c\rangle. \\
\end{align*}
Using induction we can now prove that for $c \in int(dom f)$ and for every $n \ge 0$ we have that \[ D_f(c,x^0) = D_f(c,x^n) + \sum_{k=1}^n\left( D_f(x^k,x^{k-1}) + \langle q^{p(k)}, x^{p(k)} - x^k\rangle \right)+ \sum_{k \in \mathcal{C}} \langle x^{p(k,n)} - c,q^{p(k,n)} \rangle. \]

Using Equation \ref{eq:2}, we have that all the inner products on the right hand of the above equation are non-negative. Since the Bregman distances are always non-negative, every term on the right hand side in non-negative. Thus, if we take the limit as $n \to \infty$ of the equation we get that $D_f(c,x^n)$ is a convergent sequence. Thus, is bounded. We also get that $S_n = \sum_{k=1}^n D_f(x^k,x^{k-1})$ is also a convergent sequence and hence is bounded. Thus, we get that \[ \lim_{n \to \infty} D_f(x^n,x^{n+1}) = 0. \]

Then using the previous reasoning we see that $x^n$ is a bounded sequence and has accumulation points. Furthermore by \cite{dykstra} we have that all the accumulation points lie inside $dom f$ and we have that $x^k - x^{k-1} \to 0$. 

Now using Equation \ref{eq:4} and adding and subtracting $x^{p(k,n)}$ from $c-x^n$, we get that \begin{equation} \langle c-x^n, \nabla f(x^0) - \nabla f(x^n) \rangle = \underbrace{\sum_{k \in \mathcal{C}} \langle c - x^{p(k,n)}, q^{p(k,n)} \rangle}_{S_1(n)} + \underbrace{\sum_{k \in \mathcal{C}} \langle x^{p(k,n)}- x^n, q^{p(k,n)} \rangle}_{S_2(n)}. \end{equation}

Using Equation \ref{eq:2}, we see that $S_1(n)$ is non positive. Now we want to prove that \begin{equation} \label{eq:6} \lim_n \sum_{k \in \mathcal{C}} |\langle x^{p(k,n)}- x^n, q^{p(k,n)} \rangle| \stackrel{?}{=} 0. \end{equation}

Let $K := cl(conv(\{x^n\}))$. That is, $K$ is the smallest closed convex set containing the sequence $\{x^n\}$. Then we have that \[ K = conv(cl(x^n)) = conv(\{x^n\} \cup \{\text{ cluster points of } x^n\}) \subseteq conv(int(dom f)) = int(dom f). \]
The first equality comes from the fact that the cluster points of $x^n$ are exactly the points needed to make the set $\{x^n\}$ closed. The second comes from the fact that the clusters points $x^n$ are in the interior of the domain of $f$, as we previously showed. Finally, since we assumed that the domain of $f$ is convex, we get the last equality. 

Then, since $f$ is very strictly convex and $K$ is compact (we showed that the sequence $\{x^n\}$ is bounded), we can use Fact \ref{fact:4} to see that there exists $0< \theta$ and $\Theta < +\infty$, such that \[ D_f(x^k,x^{k-1}) \ge \theta \|x^k - x^{k-1}\|^2 \text{ and } \|\nabla f(x^k) - \nabla f(x^{k-1}) \| \le \Theta \|x^k-x^{k-1}\|. \]
Then using the fact that $\sum_{k=1}^\infty D_f(x^k,x^{k-1}) < \infty$ we have that \[ \sum_{k=1}^\infty \|x^k -x^{k-1}\|^2 \le \frac{1}{\theta}\sum_{k=1}^\infty D_f(x^k,x^{k-1})  < \infty. \]
Now consider the following telescoping identity
\[ q^n = (q^n - q^{p(n)}) - (q^{p(n)} - q^{p(p(n))}) - \hdots - (q^{p^k(n)} - q^{p^{k+1}(n)}) - \hdots 0. \]
Then, we get that \[ \sum_{k \in \mathcal{C}} \|q^{p(k,n+1)} \| \le \sum_{k=1}^n \|q^k - q^{p(k)}\| = \sum_{k=1}^n \|\nabla f(x^k) - \nabla f(x^{k-1})\| \le \Theta \sum_{k=1}^n \|x^k-x^{k-1}\|. \]

Now note that we only forget a constraint if $q^k = 0$ and that in each iteration we look at most $2M$ constraints (where $M$ is number of constraints). Thus, if $q^{p(k,n)} \neq 0$ then, we must have that $p(k,n) \ge n - 4M$. That is, we either have projected onto that constraint in the current iteration or if we haven't projected onto that constraint yet in the current iteration, we must have projected onto it in the previous iteration. Thus, using the above we get that 

\begin{align} \sum_{k \in \mathcal{C}} |\langle x^{p(k,n)}- x^n, q^{p(k,n)} \rangle &\le \sum_{k \in \mathcal{C}} \| x^{p(k,n)}- x^n\|\| q^{p(k,n)}\| \\
&\le  \sum_{k \in \mathcal{C}} \|q^{p(k,n+1)} \|  \sum_{k = n-8M}^n \|x^{k}-x^{k-1}\| \label{eq:nottrue} \\
&\le \Theta \sum_{k=1}^n \|x^k-x^{k-1}\| \sum_{k = n-4M}^n \|x^{k}-x^{k-1}\|. \end{align}

Then using proposition 3.1 from \cite{dykstra} with the fact that $ \sum_{k=1}^\infty \|x^k -x^{k-1}\|^2 < \infty $ we get that the limit of the right hand side is 0. Thus, Equation \ref{eq:6} is true. Now looking at Equation \ref{eq:6} again we see that there exists a subsequence $k_n$ such that \[ \sum_{k \in \mathcal{C}} |\langle x^{p(k,k_n)} - x^{k_n}, q^{p(k,k_n)} \rangle | \to 0 \text{ and } 0 \ge \limsup_n \langle c-x^{k_n}, \nabla f(x^0) - \nabla f(x^{k_n}) \rangle. \]
Let $x^*$ be the accumulation point for this convergent subsequence (we may need to pass to another subsequence). Now we want to show that $x^*$ is feasible.

Assume for the sake of contradiction that $x^*$ is not feasible. Let $\tilde{C}$ be the set of constraint sets that $x^*$ belongs to, i.e., the set of constraints $x^*$ satisfies. Recall that $C$ is the set of all constraints, and let \[
    \epsilon = \phi\left(\inf_{\hat{C} \in C \cap \tilde{C}^c} \dist(x^*,\hat{C})\right). \]
Let $N$ be such that for all $k \ge N$ we have that $\|x^{n_k} - x^*\| < \epsilon/2$. Then by Equation 6.48 from \cite{book}, since $x^n$ is bounded, $x^{n_k} \to x^*$, and $D(x^{n+1}, x^n) \to 0$ we have that for any $t$, 
\[ x^{n_k + t} \to x^*. \] Thus, in particular, we see that for all $t \le 2|\tilde{C}| + 2M +1 =: T$ this is true. 

Consider our augmented subsequence $x^{n_k},x^{n_k+1},\hdots, x^{n_k+T}$ (note if $n_{k+1}-n_k \to \infty$ then this is not all points in the sequence). We want to show that infinitely often one of these values satisfies a constraint not in $\tilde{C}$. Then since we have finitely many constraints, at least one constraint $\tilde{c}$ not in $\tilde{C}$ must be satisfied infinitely often by our augmented sequence. Then since our augmented sequence converges to $x^*$ and we are only looking at closed constraints, we must have $x^*$ satisfies this constraint. Which is a contradiction. Thus, $x^*$ must be feasible.

Let $k > N$. Then we know that $x^{n_k}$ is the variable that we get after the $n_k$th projection. Let's assume that this happens in iteration $\nu_k$. Note that in any iteration we project onto any constraint at most twice, once if its in our list $C^{\nu_k}$ and once if its in the list $L$ returned by the oracle. Thus we have two possibilities for which iteration the $n_k + 2|\tilde{C}|+1$ projection takes place. Let us case on when that happens. 

\textbf{Case 1:} If $n_k + 2|\tilde{C}|+1$ takes places in $\nu_k$th iteration infinitely often, then in case since we project onto each constraint at most twice, by pigeonhole principle one of the projections between $n_k$ and $n_k + 2|\tilde{C}| + 1$ must be onto a constraint not represented in $\tilde{C}$. Thus, if we look amongst $x^{n_k},x^{n_k+1},\hdots, x^{n_k+2|\tilde{C}|+1}$, we must have projected onto a constraint not in $\tilde{C}$ infinitely often. 

\textbf{Case 2:} If $n_k + 2|\tilde{C}|+1$ takes places in $\nu_k+1$th iteration infinitely often. Then since our oracle returns at most $M$ constraints, we must have that amongst $x^{n_k},x^{n_k+1},\hdots, x^{n_k+T}$ we must have projected onto all of the constraints returned by the oracle during step 1 of the $\nu_k+1$th iteration. Thus, we must have projected onto some constraint $\hat{c}^{\nu_k+1}$ such that distance of $x^{\nu+1}$ to $\hat{c}^{\nu_k+1}$ is at least $\phi(d_{\nu_k+1})$ where $d_{\nu_k+1}$ is the distance from $x^{\nu_k+1}$ to our feasible region $C$. Then, we have that 
\[
    \dist(x^{\nu_k+1},\hat{c}^{\nu_k+1}) \ge \phi(d_{\nu_k+1}) \ge \phi\left(\inf_{\hat{C} \in C \cap \tilde{C}^c} \dist(x^{\nu_k+1},\hat{C})\right). \]
Noting that $x^{\nu+1}$ is within our augmented sequence and taking the limit, we get that \[
\lim_{k \to \infty} \dist(x^{\nu_k+1} , \hat{c}^{\nu_k+1}) \ge \epsilon.\] 
Since, we have finitely many constraints, we may assume (by passing to a subsequence) that all $\hat{c}^{\nu_k+1}$ are equal to some constraint $\hat{c}$. Then, we have that \[ \dist(x^*, \hat{c}) \ge \epsilon. \]
Thus, $\hat{c} \not\in \tilde{C}$. However, we now project onto $\hat{c}$ infinitely often. Hence $x^*$ must satisfy the constraint $\hat{c}$. Hence we have a contradiction. Thus, $x^*$ is feasible. 

Finally, we have that since $D_f(\cdot, \cdot)$ is separately continuous, we have that $D_f(c,x^{k_n}) \to D_f(c,x^*)$ and $D_f(x^{k_n},x^0) \to D_f(x^*,x^0)$. Thus for an arbitrary feasible $c$ we have that 

\begin{align*}  \langle c-x^*, \nabla f(x^0) - \nabla f(x^{*}) \rangle &= D_f(c,x^*) + D_f(x^*,x^0) - D_f(c,x^0) \\
&= \lim_n D_f(c,x^{k_n}) + D_f(x^{k_n},x^0) - D_f(c,x^0) \\
&= \lim_n \langle c-x^{k_n}, \nabla f(x^0) - \nabla f(x^{k_n}) \rangle \\
&\le 0. \end{align*}

Then by Fact \ref{fact:3}, we have that $x^* = P_C^f(x^0)$. Thus, all accumulation points are optimal points. Thus, by strict convexity of $f$, we see that this is the unique solution. 

Now to get the unique solution to constrained minimization problem we need to initialize $x^0$ such that $\nabla f(x^0) = 0$. This is because $x^*$ is the point that minimized \[ D_f(x,x^0) = f(x) - f(x^0) - \langle \nabla f(x^0), x-x^0 \rangle. \] Thus, if the last term is 0, then second term is a constant and we minimize the first term which is what we want. 

\subsection{Convergence Rate for Quadratic Objective Function}

\begin{lem} For any $C_i$ such that $x^* \in int(C_i)$ we have that we only look at this constraint finitely often. \end{lem}
\begin{proof} Every time we look at a constraint we project onto its boundary. Thus, if we project onto $C_i$ infinitely often $x^*$ cannot be in the interior of $C_i$. 
\end{proof}

Let $I$ be the set of all active constraints. That is if $x^*$ is the optimal solution then \[ I = \{ i : x^* \in \partial C_i\}\]

Let $S$ be the set of all $x$ that satisfy these constraints (namely $S = \{ x : \forall i \in I, x \in C_i \}$). Let $B_x$ be the boundary for a constraint in $I$ furthest from $x$ and define \[ \mu = \inf_{x \not\in S} \frac{d(x, B_x)}{d(x,S)} > 0 \]

Let $U$ be the set of all optimal duals $q's$ that is $U = \{ (q_1, \hdots, q_N) :\nabla f(x^0) - \nabla f(x^*) = \sum_i q_i \}$ and let $I_\nu =  \{ i : q^{p(i,\nu+1)} \neq 0\}$.

\begin{lem} \label{lem:converse-general} For any sequence $x^n \to x^*$ where $x^*$ is the optimal for an instance of problem \ref{problem:convex}, and $x^n,q^n$ maintain the KKT conditions then there exists an $M'$ such that for all $\nu \ge M"$ there exists a $z \in U$ such that for all $i \not\in I_\nu$  we have that $q_i = 0$ \end{lem}
\begin{proof} Let $V_\nu = \{ (q_1, \hdots, q-N) : \forall i \not\in  I_\nu, q_i = 0\}$. Then assume for the sake of contradiction that the result is false. Thus, there is a sequence $\nu_k$ such that $V_{\nu_k} \cap U = \emptyset$. Then since there finitely many different $I_\nu$ (hence finitely many $V_\nu$) we have that one of these must occur infinitely often. Thus, by taking an appropriate subsequence we assume that without loss of generality that all $I_{\nu_k}$ are all equal. Thus, let $V=V_{\nu_k}$. Thus, we have that $V \cap U = \emptyset$. 

Then since $V$ is a subspace (hence closed), $U$ is a closed set, and they are disjoint, we must have that $d(V,U) > 0$. But now $(q^{p(1,\nu_k+1)}, \hdots, q^{p(N,\nu_k+1)}) \in V$. Thus, \[ d((q^{p(1,\nu_k+1)}, \hdots, q^{p(N,\nu_k+1)}) , U) \ge d(V,U) > 0\] By Theorem \ref{thm:convergence-general} we have that $x^\nu \to x^*$. Thus, for any $(q_1,\hdots, q_N) \in U$, \[ \sum_iq^{p(i,\nu_k+1)} = \nabla f(x^0) - \nabla f(x^\nu)  \to \nabla f(x^0)- \nabla f(x^*) =\sum_i q_i\] Thus $d((q^{p(1,\nu_k+1)}, \hdots, q^{p(N,\nu_k+1)}), U) \to 0$ which is a contradiction \end{proof}

The rest of the proof is same as for Theorem \ref{thm:linear}. As before we get the following corollaries.

\begin{cor} If we further have that there exists an $x^0 \in dom f$ such that $\nabla f(x^0) = 0$. Then we have that the Bregman projection of $x^0$ onto $C$ is the solution to Problem \ref{problem:convex}. \end{cor}

\begin{rem} Similar to before while this gives linear convergence but $\rho$ is very close to one. As we will see later $\rho \le \frac{F}{F+1}$ where $F$ is number of convex sets whose boundary the optimal solution lies on that are seen by the algorithm.  \end{rem}

\begin{rem} Unlike in the linear case, our proof does not guarantee that a truly stochastic version will converge to the optimal solution. Equation \ref{eq:nottrue} does not hold in this case.  \end{rem}

\newpage

\bibliography{citations}

\end{document}